\documentclass[11pt]{article}

\usepackage[margin=1in]{geometry}
\usepackage[T1]{fontenc}
\usepackage[utf8]{inputenc}
\usepackage{lmodern}
\usepackage{color}
\usepackage{natbib}
\usepackage{graphicx}
\usepackage{caption}
\usepackage{subcaption}
\usepackage{mathtools}
\usepackage{amsthm}
\usepackage{amsmath}
\usepackage{amssymb}
\usepackage{booktabs}
\usepackage{float}
\usepackage[colorlinks=true,linkcolor=blue,citecolor=blue,urlcolor=blue]{hyperref}
\usepackage{cleveref}

\definecolor{darkgreen}{rgb}{0,0.5,0}
\definecolor{purple}{rgb}{1,0,1}

\theoremstyle{plain}
\newtheorem{theorem}{Theorem}
\newtheorem{proposition}[theorem]{Proposition}
\newtheorem{assumption}[theorem]{Assumption}

\newtheorem{corollary}[theorem]{Corollary}
\theoremstyle{definition}

\theoremstyle{remark}

\newcommand{\reals}{\mathbb{R}}
\newcommand{\naturals}{\mathbb{N}}
\newcommand{\integers}{\mathbb{Z}}
\newcommand{\complex}{\mathbb{C}}

\newcommand{\Hcal}{\mathcal{H}}

\newcommand{\Vcal}{\mathcal{V}}

\newcommand{\Tcal}{\mathcal{T}}

\DeclareMathOperator*{\expect}{\mathbb{E}}

\DeclarePairedDelimiter\floor{\lfloor}{\rfloor}

\newcommand{\indicator}{\mathbf{1}}

\newcommand{\identity}{\mathbf{I}}
\newcommand{\imaginary}{\operatorname{i}}
\newcommand{\torus}{\mathbb{T}}
\newcommand{\norm}[1]{\left\lVert#1\right\rVert}
\newcommand{\inner}[2]{\left\langle #1, #2 \right\rangle}

\begin{document}

\title{Operator Learning for Schr\"{o}dinger Equation: Unitarity, Error Bounds, and Time Generalization}
\author{
\begin{tabular}[t]{c}
\textbf{Yash Patel}\thanks{Equal Contribution.}\\
University of Michigan\\
\texttt{ypatel@umich.edu}
\end{tabular}
\and
\begin{tabular}[t]{c}
\textbf{Unique Subedi}\footnotemark[1]\\
University of Michigan\\
\texttt{subedi@umich.edu}
\end{tabular}
\and
\begin{tabular}[t]{c}
\textbf{Ambuj Tewari}\\
University of Michigan\\
\texttt{tewaria@umich.edu}
\end{tabular}
}
\date{}

\maketitle

\begin{abstract}

We consider the problem of learning the evolution operator for the time-dependent Schr\"{o}dinger equation, where the Hamiltonian may vary with time. Existing neural network-based surrogates often ignore fundamental properties of the Schr\"{o}dinger equation, such as linearity and unitarity, and lack theoretical guarantees on prediction error or time generalization. To address this, we introduce a linear estimator for the evolution operator that preserves a weak form of unitarity. We establish both upper bounds and lower bounds on the prediction error of the proposed estimator that hold uniformly over classes of sufficiently smooth initial wave functions. Additionally, we derive time generalization bounds that quantify how the estimator extrapolates beyond the time points seen during training. Experiments across real-world Hamiltonians- including hydrogen atoms, ion traps for qubit design, and optical lattices- show that our estimator achieves relative errors up to two orders of magnitude smaller than state-of-the-art methods such as the Fourier Neural Operator and DeepONet.
\end{abstract}

\section{Introduction}

Solving the time-dependent Schr\"{o}dinger equation is of interest in various engineering applications, including material science \citep{liu2022algorithm, li2020real} and the design of quantum computers \citep{mohammed2024optical,chen2024simulating}. In all but the simplest cases, solving this equation requires numerical methods, which are computationally expensive even for a single initial condition \citep{peskin1993solution, van2011efficiency}. However, a single simulation is often insufficient for practical applications. For example, in qubit design, an essential requirement is that the system maintains coherence across a range of environmental conditions. This necessitates repeated simulations under varying initial conditions and system parameters, requiring significant computational resources \citep{nagele2023decoherence, wu2014rabi}. 

In fact, the need for repeated PDE solutions is widespread in engineering design, such as in simulating the Navier-Stokes equations for evaluating car or airfoil designs \citep{shahrokhi2007airfoil, eyi1994airfoil}. To address this, operator learning has emerged as a promising approach for the surrogate modeling of PDEs, allowing efficient computation in such cases where repeated evaluation is required \citep{azizzadenesheli2024neural,augenstein2023neural}. Building on this idea, recent works have explored operator learning for accelerating solutions to the time-dependent Schr\"{o}dinger equation \citep{mizera2023scattering, zhang2024artificial, shah2024fourier}. These works, however, typically use general-purpose neural operators \citep{kovachki2023neural}, most commonly the Fourier Neural Operator (FNO) \citep{li2020fourier}, without explicitly using the special structures of the Schr\"{o}dinger evolution, such as linearity and unitarity.
In related fields, it has been demonstrated that incorporating known physical priors is often crucial for effective surrogate learning in data-scarce settings \citep{batzner20223, merchant2023scaling}.

\begin{figure}
    \centering
    \begin{subfigure}{0.32\textwidth}
        \centering
        \includegraphics[width=\linewidth]{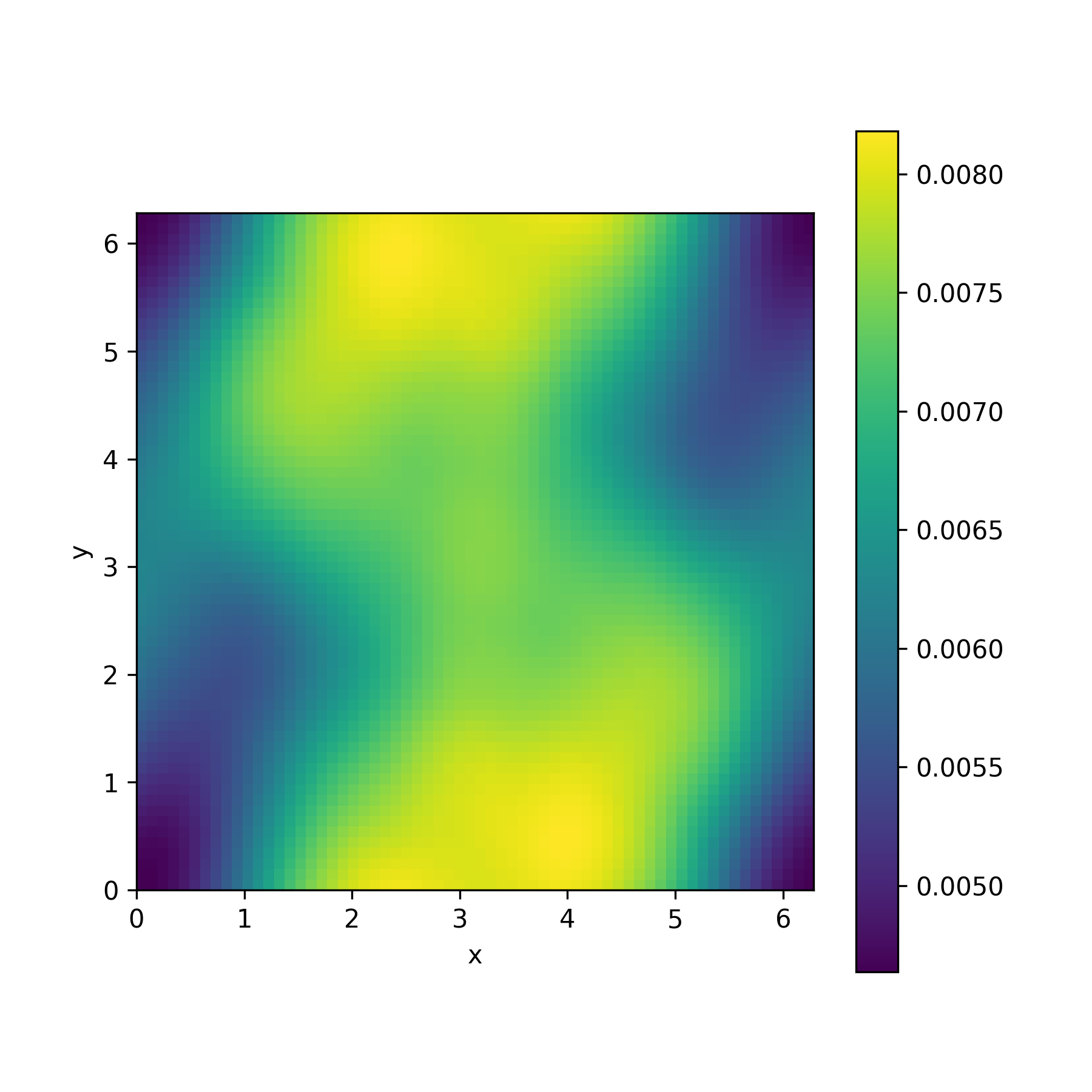}
        \caption{Initial Wave }

    \end{subfigure}
    \begin{subfigure}{0.32\textwidth}
        \centering
        \includegraphics[width=\linewidth]{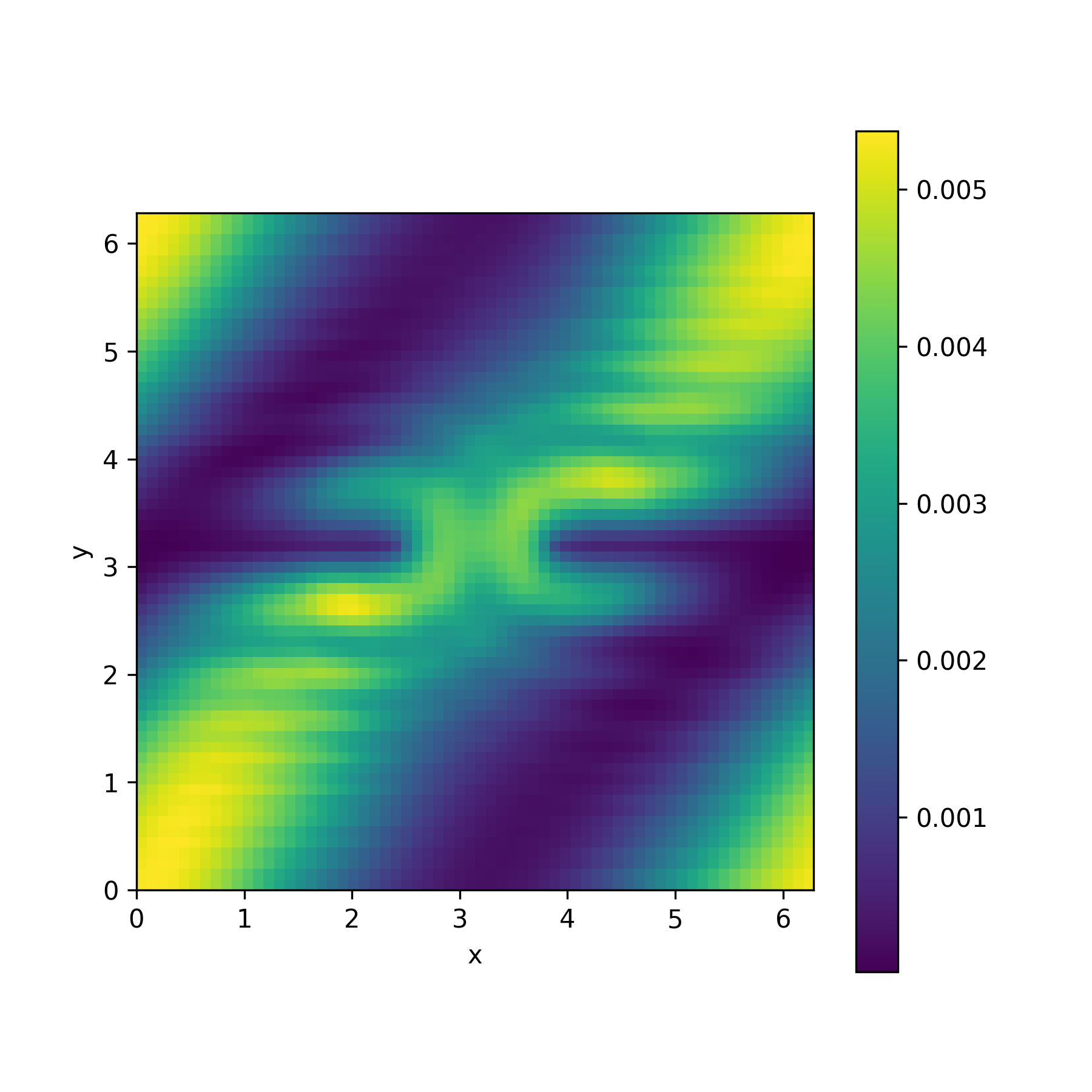}
        \caption{True Wave at $T=0.1$}
    \end{subfigure}
    \begin{subfigure}{0.32\textwidth}
        \centering
        \includegraphics[width=\linewidth]{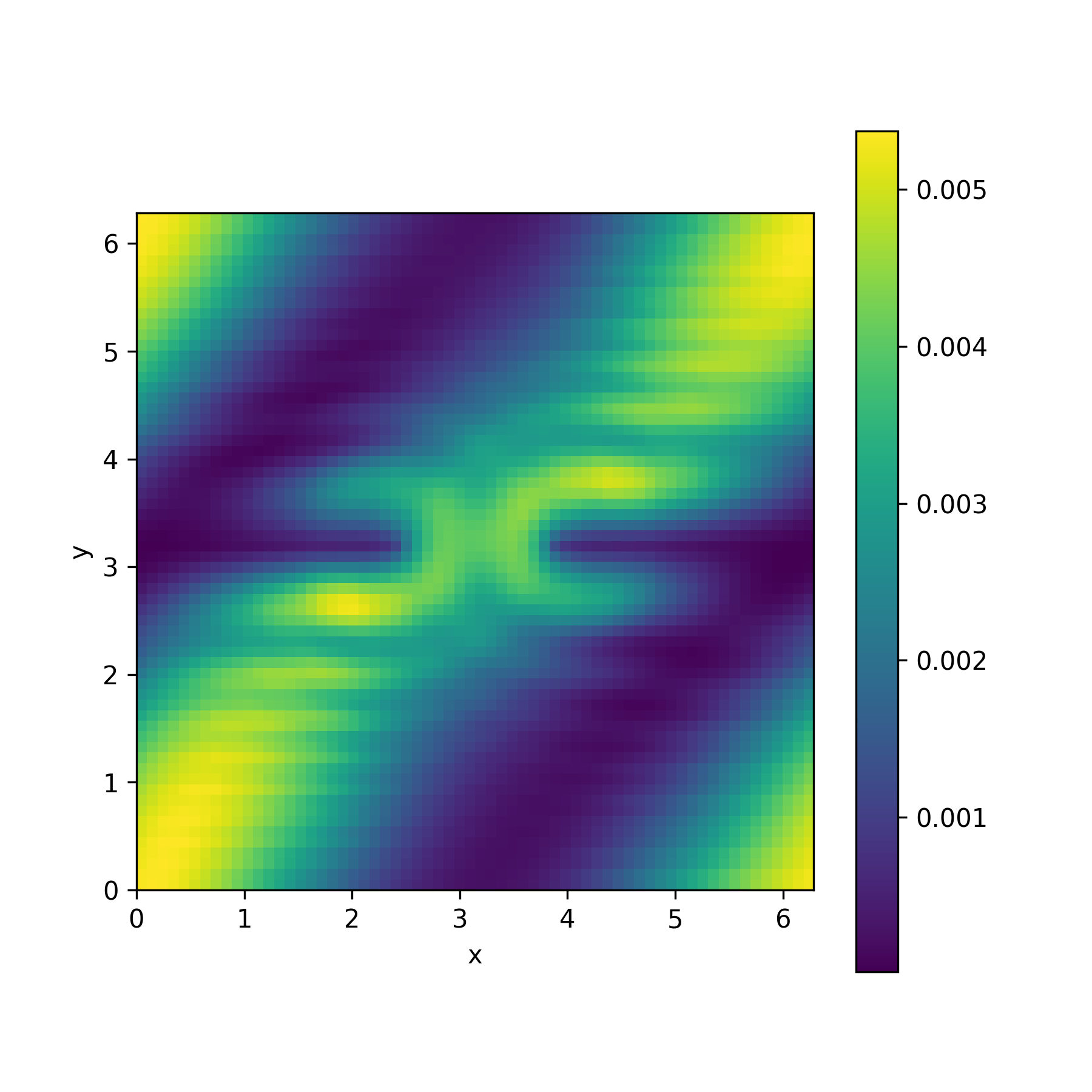}
        \caption{Our Estimator's Prediction }
    \end{subfigure}
    
    \caption{Squared amplitude $ |\psi(x)|^2 $ of the initial wave, the true wave at $ T=0.1 $, and the estimator's prediction for the barrier potential with double slits on $[0,2\pi)^2$.}
    \label{fig:estimator}
\end{figure}

Thus, in this work, we propose a surrogate model for the time-dependent Schr\"{o}dinger equation that exploits the fundamental structures of this equation. Specifically, let $\operatorname{F}$ be the true evolution operator of the Schr\"{o}dinger equation that maps the initial wave function $\psi$ to its evolved state at some fixed time $T>0$. We introduce an active data collection strategy and a \emph{linear estimator} $\widehat{\operatorname{F}}$ to approximate $\operatorname{F}$ and establish the following. 
\begin{itemize}
    \item[(i)] \textbf{(Empirical Evaluation)} We evaluate $\widehat{\operatorname{F}}$
  across a range of Hamiltonians, including hydrogen atoms, double-slit potentials (see Figure \ref{fig:estimator}), an ion trap used in qubit design, and optical lattices. We compare our estimator, trained using actively collected data, with standard neural operator baselines such as the Fourier Neural Operator (FNO) and DeepONet, which are trained on passively collected data. This active-passive mismatch arises due to the lack of established active data collection strategies for neural operator models. Across all settings, our estimator consistently achieves relative errors up to two orders of magnitude smaller than these baselines.

  \item[(ii)] \textbf{(Preservation of Physical Laws)} We prove that the surrogate $\widehat{\operatorname{F}}$ preserves a weak form  of unitarity, a fundamental property of the Schr\"{o}dinger equation. In particular, the estimator is exactly unitary on the queried subspace and non-expansive on the entire space. Thus, our work aligns with broader efforts in physics-informed learning to incorporate known physical symmetries and conservation laws into model design and training \citep{li2024physics, richter2022neural}.

  \item[(iii)] \textbf{(Theoretical Guarantees)} We prove upper bounds and matching lower bounds for the prediction error of $\widehat{\operatorname{F}}$, where these bounds characterize the performance of the proposed estimator rather than minimax optimality. These bounds establish tight rates in terms of the number of training samples, the accuracy of the PDE solver used to generate training data, and the Sobolev smoothness of the initial wave function. Instead of bounding expected error with respect to a distribution, as is common in statistical learning theory, we establish a stronger guarantee that holds uniformly over smooth initial waves.

  \item[(iv)]  \textbf{(Time Generalization)} We establish time generalization bounds, showing that $ \widehat{\operatorname{F}} $ trained on data up to time point $T$ can extrapolate to future time points $ T' > T $ when the potential is time-independent and sufficiently smooth. While time generalization has been studied empirically in prior works on the Schr\"{o}dinger equation \citep{mizera2023scattering, shah2024fourier}, little is known theoretically about when such generalization is possible. To the best of our knowledge, this work provides the first mathematical guarantees for time generalization in operator learning.
\end{itemize}

\subsection{Related Works on Operator Learning and Time-Dependent Schr\"{o}dinger Equation} 

 One of the earlier works in this area was by \cite{mizera2023scattering}, who used Fourier Neural Operators (FNOs) by \cite{li2020fourier} to estimate the evolution operator for simple quantum systems, including random potential functions and the double-slit potential. Additionally, they studied the time generalization ability of the learned operator by evolving the wave function beyond the time range included in the training dataset. Notably, their learned operator is more general as it maps from the initial wave function and potential function to the evolved state, rather than learning a propagator for a fixed Hamiltonian. A related work by \cite{niarchos2024learning} considers learning phases of amplitudes in scattering problems. While these works primarily study isolated quantum systems, \cite{zhang2024artificial} and \cite{zhang2025neural} used FNO-based architectures to study dissipative quantum systems, where the system interacts with the surrounding environment and can be driven by external fields. They also assessed the time generalization capacity of the learned propagator.  Recently, \cite{shah2024fourier} used FNOs to learn the evolution operator of quantum spin systems and studied their ability for both single-step and multi-step time generalization.

We also note the related work on learning unitary transformations between large but finite-dimensional Hilbert spaces \citep{bisio2010optimal, hyland2017learning, belov2024partially}, a problem that has been studied extensively in areas ranging from quantum state simulation \citep{johansson2012qutip} to quantum computing \citep{huang2021learning}. There is also a growing body of work on learning linear operators from data in the context of PDE modeling \citep{de2023convergence, mollenhauer2022learning, subedi2024benefits}, which is most closely related to the perspective adopted in this work. Finally, our linear estimator shares some conceptual similarity with the exact diagonalization methods commonly used to numerically solve the Schr{\"o}dinger equation \citep{lin1993exact}. However, unlike exact diagonalization, which explicitly diagonalizes the Hamiltonian matrix, our approach defines an estimator for the solution operator in an arbitrary basis without performing any diagonalization. A more detailed discussion of related works is provided in Appendix \ref{appdx:ext-related-works}.


\section{Preliminaries}


Let $ \mathbb{R} $ and $ \mathbb{C} $ denote the real and complex numbers, respectively, while $ \mathbb{N} $ and $ \mathbb{Z} $ represent the natural numbers and integers. Define $\naturals_0 := \naturals \cup \{0\}$. For any $ x \in \mathbb{R}^d $, the $ \ell^p $-norm is denoted by $ |x|_p $. For  $ z \in \mathbb{C} $ with $ z = a + b i $, define $
|z| = \sqrt{a^2 + b^2}.$ For $ x, y \in \mathbb{R}^d $, the Euclidean inner product is denoted by $ x \cdot y $.  For two non-negative functions $f$ and $g$ defined on $\naturals$, we say $f \lesssim g$ if there exists a universal constant $c$ and $n_{0} \in \naturals$ such that $f(n) \leq c g(n)$ for all $n \geq n_0$. Equivalently, $f \gtrsim g$ means there exists a constant $c$ and $n_0 \in \naturals$ such that $f(n) \geq c g(n)$ for all $n \geq n_0$. Finally, $f \asymp g$ means we have both $f \gtrsim g$ and $f \lesssim g$.

Let $\Omega \subset \reals^d$ be a bounded set. Define  $
L^2(\Omega) := \left\{ u: \Omega \to \mathbb{C} \,  | \, \int_{\Omega} |u(x)|^2 \, dx < \infty \right\}$ to be the set of squared-integrable functions on $\Omega$.
This is a Hilbert space with the inner product  $
\langle u, v \rangle_{L^2} = \int_{\Omega} u(x) \, \overline{v(x)} \, dx,$
where $ \overline{z} = a - b i $ is the complex conjugate of $ z = a + b i $. The norm induced by this inner product is denoted by $ \| \cdot \|_{L^2} $.

\subsection{Time-Dependent Schr\"{o}dinger Equation }
The time-dependent Schr\"{o}dinger equation states that the quantum system evolves as
\begin{equation*}
    \imaginary \hbar \,\,  \partial_{t} \psi(\cdot, t)= \operatorname{H}(t) \,\psi(\cdot, t).
\end{equation*}
Here, $\psi(\cdot, t): \Omega \subseteq \reals^d \to \complex$ is the wave function at time point $t$, and $\operatorname{H}(t)$ is the Hamiltonian operator that can evolve over time. For a single particle system, we generally have $d\in \{1,2,3\}$, and the Hamiltonian can be written as   
$
\operatorname{H}(t) = -\frac{\hbar^2}{2m} \Delta + V_t(\cdot),
$
where $V_t(\cdot)$ is a time-dependent potential function and $\Delta$ is the Laplacian operator defined as $\Delta f := \sum_{j=1}^d \frac{\partial^2 f}{\partial x_j^2}$ for Euclidean coordinates.

The solution operator, also referred to as the time evolution operator, of the Schr\"{o}dinger equation takes the form $\Tcal \left[ \exp\left(- \frac{\imaginary}{\hbar} \, \int_{t_0}^{t} \operatorname{H}(s)\, ds \right) \right].$
That is, given an initial wave $\psi(\cdot, t_0)$ at time $t_0$, the wave function at time point $t> t_0$ is given by $\psi(\cdot, t) = \Tcal \left[ \exp\left(- \frac{\imaginary}{\hbar} \int_{t_0}^{t} \operatorname{H}(s)\, ds \right) \right]\psi(\cdot, t_0). $
Here, $ \Tcal $ denotes the time-ordering operator, which ensures that Hamiltonians at different time points are applied in the correct temporal sequence. This ordering is necessary because the time-varying Hamiltonians do not generally commute at different times. 
One can formally define this operator in terms of a Dyson series expansion \citep[Chapter 2]{sakurai2020modern}. When the Hamiltonian is constant over time, $\operatorname{H}(t)=\operatorname{H}$, the evolution operator has a more familiar form $\exp\left(-  \frac{\imaginary (t-t_0)}{\hbar}  \operatorname{H} \right).$ While the solution operator has a closed-form representation, it is generally impractical for computation, as evaluating operator exponentials requires computing higher-order powers of the Hamiltonian and summing an infinite series in its Taylor expansion.

\subsection{Problem Formulation and Goal}

Define $t_0=0$ without loss of generality. For a fixed time $T>0$, we seek to learn the operator  
\[
\operatorname{F}:= \Tcal \left[ \exp\left(- \frac{\imaginary}{\hbar} \int_{0}^{T} \operatorname{H}(s)\, ds \right) \right].
\]
Note that $\operatorname{F}$ is linear, unitary on $L^2(\Omega)$. Given a collection of $ n $ observed samples $ \{(\psi_j(\cdot, 0), \psi_j(\cdot, T))\}_{j=1}^n $, where each sample satisfies  $
\psi_j(\cdot, T) = \operatorname{F}\big(\psi_j(\cdot, 0) \big),
$
our goal is to construct an estimate $ \widehat{\operatorname{F}}_n $ that approximates $ \operatorname{F} $ accurately under a suitable metric. More precisely, we want to develop a data generation strategy and an estimation rule such that the error of the estimator, 
\[\sup_{\psi \in \Vcal} \, \|\widehat{\operatorname{F}}_n(\psi) - \operatorname{F}(\psi)|_{L^2}, \]
is small. Here, $\Vcal $ is some suitable subset of $L^2(\Omega)$ (consisting of, say, Sobolev-type smooth functions) that will be specified later. Note that our goal is to provide a \emph{uniform error bound} over $\Vcal$ rather than a bound on expected error that is more common in statistical learning theory. 
This stronger guarantee is possible in our setting because we can evaluate $\operatorname{F}(\psi)$ for any chosen initial wave function $\psi$, up to certain numerical accuracy, using a PDE solver. This flexibility allows us to actively select the most informative queries for constructing the estimator, rather than relying on the i.i.d. sampling typically used in statistical learning.

\section{Data Collection Strategy and Estimator }\label{sec:method}
In this section, we introduce our data generation strategy and define our estimation rule. Then, we establish the unitarity of the proposed estimator.  For clarity of exposition, we assume that $ \Omega = \mathbb{T}^d $, the $ d $-dimensional torus \cite[Chapter 3]{grafakos2008classical}. Extension to non-periodic domains is straightforward and is discussed in Section \ref{sec:beyond-periodic}. Throughout this paper, we identify $ \mathbb{T}^d $ by $ [0,1]^d $ equipped with periodic boundary conditions. Before moving forward, we first define the Sobolev space on $ \mathbb{T}^d $. The theoretical guarantees established in this work apply only to initial waves that belong to these Sobolev spaces on $ \mathbb{T}^d $ or their equivalent counterparts in more general domains $\Omega$.

For any $ k \in \mathbb{Z}^d $, define the function $ \varphi_k: \mathbb{T}^d \to \mathbb{C} $ by  $\varphi_k(x) := e^{2\pi i k \cdot x}.$
Using these functions, for any $s>0$, we can define the Sobolev space as 
\[
\mathcal{H}^s(\mathbb{T}^d) := \left\{ f \in L^2(\mathbb{T}^d) \quad  : \quad \sum_{k \in \mathbb{Z}^d} (1+ |k|_{2}^{2})^s \, |\langle f, \varphi_k \rangle_{L^2}|^2 < \infty \right\}.
\]
This space is equipped with the norm  $\|f\|_{\mathcal{H}^s} := \sqrt{\sum_{k \in \mathbb{Z}^d} (1+ |k|_{2}^{2})^s \, |\langle f, \varphi_k \rangle_{L^2}|^2}.$
Although $ \mathcal{H}^s(\mathbb{T}^d) $ is a Hilbert space, we will not rely on its Hilbertian properties in this work. When $s \in \naturals $, this definition of the Sobolev space is equivalent to its formulation based on derivatives. That is, $\mathcal{H}^s(\mathbb{T}^d) \simeq \left\{ f \in L^2(\mathbb{T}^d) \, \Big |\, \sum_{|\alpha|_{1}\leq s} \norm{\operatorname{D}^{\alpha}f}_{L^2}^2 < \infty \right\},$
where  $\operatorname{D}^{\alpha}$ denotes the differential operator. For more details on this equivalence, we refer the readers to \cite[Chapter 4.3]{taylorpartial}. 
The Sobolev smoothness assumption ensures that the Fourier coefficients of $f$ decay sufficiently fast, so that $f$ is well-approximated by finitely many Fourier modes. This allows accurate reconstruction of the evolution operator from a finite set of actively queried samples.

\subsection{Estimator}\label{sec:estimator}

Ideally, one would construct the estimator in a basis that diagonalizes the full Hamiltonian operator,
$\operatorname{H}(t) \propto - \Delta + V_t(\cdot)$.
However, such a basis is generally difficult to compute and depends on the specific potential $V_t$.
We therefore use the Fourier basis, which diagonalizes the kinetic component $-\Delta$ and provides a natural, problem-independent basis on the torus.

Given a sample size budget of $ n $ such that $n \geq 3^d$, define  $
K_n := (n^{\frac{1}{d}} - 1)/2.$
The estimator is constructed by querying a PDE solver to obtain  $
w_k = \operatorname{P}(\varphi_k)$
for each $ k \in \mathbb{Z}^d $ such that $ |k|_{\infty} \leq K_n $, where $ \operatorname{P} $ denotes the numerical PDE solver for the corresponding solution operator $\operatorname{F}$ of the Schr\"{o}dinger equation. Since the number of such frequency indices satisfies  $
|\{ k \in \mathbb{Z}^d : |k|_{\infty} \leq K_n \}| = (2\floor{K_n}+1)^d \leq n,$
the sample budget of $n$ is not exceeded. Using the labeled dataset $ \{(\varphi_k, w_k)\}_{|k|_{\infty} \leq K_n} $, we then define the estimator  
\begin{equation}\label{eqn:estimator}
\widehat{\operatorname{F}}_n := \sum_{k \in \integers^d \, :\, |k|_{\infty} \leq K_n} w_k \otimes \varphi_k.
\end{equation}
Here, $ w \otimes v $ denotes a rank-one operator, defined for any $ u \in L^2(\torus^d) $ as  $(w \otimes v)(u) = w \langle u, v\rangle_{L^2}.$ 
Note that this estimator naturally extends to general domains $ \Omega $. For general $\Omega$, one can query the eigenfunctions of the Laplacian operator. Recall that $ \varphi_k $ are the eigenfunctions for $ \mathbb{T}^d $. For more complex domains, the estimator can also be constructed using alternative domain-specific algebraic bases, such as orthogonal polynomials or wavelets. See Appendix \ref{sec:beyond-periodic} for a more detailed discussion.

\noindent \textbf{Remark.} The estimator in Equation \ref{eqn:estimator} is closely related to the approach of \cite{subedi2024benefits}, who proposed a similar data collection strategy and estimator to highlight the advantages of active data collection in operator learning. However, their error guarantees hold only in expectation over input samples drawn from a distribution, whereas we establish a uniform bound. While our results are stated for the $\operatorname{F}$, the error rates in Section \ref{sec:error-convergence} apply to any bounded linear operator, making our work a strict generalization of \cite{subedi2024benefits}. Furthermore, using the structure of the time-dependent Schr\"{o}dinger equation, we establish additional properties such as weak unitarity and time generalization in this work, which do not necessarily hold for general linear operators.

\noindent \textbf{Comparison to Galerkin methods.} 
A related line of work from classical numerical methods is that of ``spectral Galerkin methods.'' Given the maturity of this technique, many references are available for review \cite{fletcher1984computational}; we merely provide a brief overview to highlight its distinction from the work considered herein. 

In a general setting, one wishes to solve a PDE of the form $\mathcal{L} u = f$ for a linear operator $\mathcal{L}$, where ``solve'' is understood to mean finding $u$ for a specified $f$. Spectral Galerkin methods do so by assuming a decomposition of the solution function $u = \sum_{k} c_k \varphi_k$ over some basis collection $\{\varphi_k\}$. Importantly, the ``spectral'' and ``Galerkin'' are distinct descriptors of this method. ``Spectral'' refers to $\{\varphi_k\}$ being a \textit{global} basis set, namely where $\varphi_k$ are elements defined over the \textit{entire} domain, such as Fourier or Chebyshev modes. This is in contrast to methods such as finite-elements, where the full domain is discretized with basis elements then defined over these individual elements.

The ``Galerkin'' aspect of the name refers to the particular enforcement mechanism we use to solve for the coefficients. That is, if we take a truncated approximation to be $u \approx \sum_{k=1}^{K} c_k \varphi_k$, the Galerkin approach solves for coefficients to ensure that residuals are orthogonal to each basis function. That is $\langle \mathcal{L} u - f, \varphi_k\rangle = 0 \qquad\forall\ k = 1, ..., K$. This, in turn, produces a system of equations whose solution yields the estimated coefficients $\{c_k\}_{k=1}^{K}$ of interest. 

Notably, in the spectral Galerkin method, any newly specified $f$ must solve this system of equations anew. This is the critical difference from the operator learning setting considered herein: in our setting, we wish to amortize the cost such that any newly queried $f$ can be solved without any significant computation.


\subsection{On Unitarity of the Estimator }

A key property of the Schr\"{o}dinger equation is the unitarity of $\operatorname{F}$. More precisely, $ \operatorname{F} $ is a surjective operator on $ L^2(\torus^d) $ that preserves inner products, meaning  $
\langle \operatorname{F} u, \operatorname{F} v \rangle_{L^2} = \langle u, v \rangle_{L^2}, \quad \forall u,v \in L^2(\torus^d).$ A direct consequence of this is that the $ L^2 $-norm remains invariant under evolution, $
\|\operatorname{F} (\psi)\|_{L^2}^2 = \|\psi\|_{L^2}^2.$
This property is fundamental because, when interpreting $ |\psi(x,t)|^2 $ as the probability density of a particle's position, unitarity ensures that the total probability is conserved over time. 
Moreover, as we show in Section \ref{sec:time-gen},  unitarity and its implications  plays a crucial role in controlling error accumulation during time generalization. Given this, it is natural to ask whether our estimator also satisfies unitarity.

Strictly speaking, $ \widehat{\operatorname{F}}_n $ is not fully unitary. For instance, if $ \varphi_\ell $ is a Fourier mode with $ \|\ell\|_\infty > K_n $, then $ \widehat{\operatorname{F}}_n \varphi_\ell = 0 $ despite $ \|\varphi_\ell\|_{L^2} = 1 $. However, $ \widehat{\operatorname{F}}_n $ satisfies a weaker form of unitarity, as captured in the following proposition (proof deferred to Appendix \ref{appdx:unitarity}). 


\begin{proposition}\label{prop:unitarity}
Assume that the solver preserves inner products on the queried subspace. That is,
\[
\langle \operatorname{P}(u), \operatorname{P}(v) \rangle_{L^2} = \langle u, v \rangle_{L^2}
\]
for all $u,v \in \mathrm{span}(\{\varphi_k : |k|_{\infty} \leq K_n\})$. Then the following hold.
\begin{itemize}
  \item[(i)] For all $ u, v \in \mathrm{span}(\{\varphi_k : |k|_{\infty} \leq K_n\}) $, we have 
  $\langle \widehat{\operatorname{F}}_n (u), \widehat{\operatorname{F}}_n (v) \rangle_{L^2} = \langle u, v \rangle_{L^2}.$
  \item[(ii)] For any $ u \in L^2(\torus^d) $, we have $
  \|\widehat{\operatorname{F}}_n(u)\|_{L^2} \leq \|u\|_{L^2}.$
\end{itemize}
\end{proposition}

Property (i) shows that $ \widehat{\operatorname{F}}_n $ is exactly unitary on the subspace spanned by the queried Fourier modes. However, the estimator is not surjective, since the spans of $\{\varphi_k\}$ and $\{w_k\}$ may differ, and hence it is not unitary on the full space. Property (ii) shows that $ \widehat{\operatorname{F}}_n $ is non-expansive on $L^2(\torus^d)$. That is, it acts as a contraction. This contraction property is crucial for establishing the time generalization guarantees in Section \ref{sec:time-gen}.

The inner-product preservation assumption on $\operatorname{P}$ is motivated by structure-preserving solvers of the Schr\"odinger equation, such as Crank--Nicolson and operator splitting methods, which yield unitary update operators when the discretized Hamiltonian is self-adjoint. When combined with spatial discretization and reconstruction procedures that are consistent with the $L^2$ inner product, this provides a principled justification for assuming that $\operatorname{P}$ preserves inner products on the queried subspace.

\section{Error Analysis and Convergence Rates}\label{sec:error-convergence}
A meaningful guarantee for the estimator $ \widehat{\operatorname{F}}_n $ requires a guarantee on the accuracy of the PDE solver $ \operatorname{P} $ used to approximate $ \operatorname{F} $. To that end, we impose the following assumption on the solver's accuracy.

\begin{assumption}\label{assumption}
The learner has black-box access to $ \operatorname{F} $ through an $ \varepsilon $-accurate \emph{PDE} solver $ \operatorname{P} $, which satisfies  
\[
\sup_{k \in \integers^d} \|\operatorname{P}(\varphi_k) - \operatorname{F}(\varphi_k)\|_{L^2} \leq \varepsilon.\]
\end{assumption}
\noindent We note that the assumption of $\varepsilon$-accuracy is only for inputs $\varphi_k$'s, not for arbitrary initial wave functions.

\subsection{Upper Bounds}

Under this assumption, we establish the following upper bound on the error of the estimator.
\begin{theorem}[Upper Bound]\label{thm:error}
Under Assumption \ref{assumption}, the estimator defined in Equation \eqref{eqn:estimator} satisfies  
\begin{equation*}
   \|\widehat{\operatorname{F}}_n(\psi) - \operatorname{F}(\psi)\|_{L^2} \leq    \norm{\psi}_{\mathcal{H}^s} \big(\varepsilon \, \gamma_n + 3^s \, n^{-\frac{s}{d}}\big)
\end{equation*}
for every wave function $\psi \in \Hcal^s(\torus^d)$. Here, we have
\begin{equation}\label{eq:gamma}
    \gamma_n \lesssim \begin{cases}
   1, & \text{\emph{if} } 2s > d,\\[6pt]
     \sqrt{\log{n}}  ,       & \text{\emph{if} } 2s = d,\\[6pt]
   n^{\frac{1}{2}-\frac{s}{d}} ,  & \text{\emph{if} } 2s < d.
  \end{cases}
\end{equation}

\end{theorem}



The term $\varepsilon \gamma_n$ corresponds to the irreducible error arising from inaccuracies in the PDE solver, while the term $3^s n^{-s/d}$ captures the estimation error due to restricting the estimator to finitely many Fourier modes. The latter term reflects the truncation of high-frequency components of $\psi$, whose decay is governed by its Sobolev smoothness.

The behavior of $\gamma_n$ is governed by how the solver errors accumulate across the queried modes, which in turn depends on the summability of the Fourier coefficients of $\psi$. In particular, bounding the contribution of solver noise requires controlling sums of the form
\[
\sum_{|k|_\infty \leq K_n} \left|\langle \psi, \varphi_k \rangle \right|,
\]
which introduces a transition between $\ell^2$ and $\ell^1$ summability of Fourier coefficients. When $2s > d$, the Fourier coefficients decay sufficiently fast to ensure $\ell^{1}$ summability, leading to a bounded accumulation of errors. At the critical regime $2s = d$, this summation incurs a logarithmic factor. When $2s < d$, the coefficients are no longer $\ell^{1}$ summable, and the accumulation of errors grows polynomially, yielding $\gamma_n \lesssim n^{\frac{1}{2}-\frac{s}{d}}$.

This result shows that the estimation error vanishes at the rate $n^{-s/d}$ for every $s$ and $d$. In the special case where $s > d/2$, this rate is faster than the Monte Carlo rate of $n^{-1/2}$. This improvement is due to our active data collection strategy, as described in Section \ref{sec:estimator}. This is in contrast to standard i.i.d. sampling, which cannot achieve faster than $n^{-1/2}$ convergence for metric losses. Finally, when the PDE solver is exact $(\varepsilon=0)$, we obtain the bound $\leq 3^s\, \norm{\psi}_{\mathcal{H}^s} \, n^{-s/d}$ for every $s>0$ regardless of $d$.

The proof of Theorem \ref{thm:error}, provided in Appendix \ref{appdx:thm}, proceeds via a bias-variance type decomposition of the estimator's error. The bias term arises from truncating the Fourier expansion, while the variance term reflects the accumulation of solver errors across the queried modes. Each component is then bounded using the Sobolev regularity of $\psi \in \Hcal^{s}(\torus^d)$. We also note that Sobolev-type smoothness assumptions are implicitly present in many applied works on operator learning; see Appendix \ref{sec:smoothness-appl-dicsussion} for a more detailed discussion.

\subsection{Lower Bounds}

A natural question that arises is whether these upper bounds are tight. This can be studied in two stages. First, can the bound be improved for the estimator defined in Section \ref{sec:estimator}? Second, beyond this specific estimator, does the bound remain tight when considering all possible linear surrogates for the Schr\"{o}dinger equation? Addressing the second question is more subtle, as establishing a meaningful information-theoretic lower bound requires precisely specifying the information accessible to the learner. The class of linear operators the learner is allowed to consider must be well-defined. One could argue that the PDE solver $ \operatorname{P} $ itself serves as the surrogate, or in the most extreme case, that $ \operatorname{F} $ is the optimal surrogate. Given these nuances, we leave the broader question of optimality to future work and focus here on studying the tightness of the bounds for our specific estimator $ \widehat{\operatorname{F}}_n $. The lower bound on the error of our estimator is presented in Theorem \ref{thm:error-lower} and its proof is deferred to Appendix \ref{appdx:error-lower}.


\begin{theorem}[Lower Bound]\label{thm:error-lower}  
There exists a Hamiltonian $\operatorname{H}$ such that for any sample budget $ n $ and estimator $ \widehat{\operatorname{F}}_n $ defined in \eqref{eqn:estimator} obtained by querying an $\varepsilon$-approximate PDE solver for  $\operatorname{F}$, we can find a wave function $ \psi $ with $\norm{\psi}_{\Hcal^s} \leq 2$ such that 
\[
\norm{\widehat{\operatorname{F}}_n(\psi_{\text{test}})-\operatorname{F}(\psi_{\text{test}})}_{L^2}
\gtrsim
\begin{cases}
\varepsilon + n^{-s/d}, & \mathrm{if } \, 2s>d,\\[6pt]
\varepsilon \sqrt{\log n} + n^{-s/d}, & \mathrm{if } \, 2s=d,\\[6pt]
\varepsilon \, n^{\frac12-\frac{s}{d}} + n^{-s/d}, & \mathrm{if } \, 2s<d.
\end{cases}
\]
\end{theorem}

 The lower bound matches the upper bound (up to constants) in all three regimes of $s$ and $d$. In the regime $s < d/2$, the irreducible error grows with the sample size $n$. This behavior arises because the solver errors can accumulate coherently across the queried modes. As a result, increasing the number of samples can actually worsen the overall error in this regime. This reveals a fundamental tradeoff between estimation error and irreducible error: while the estimation error decreases with $n$, the irreducible error may increase, preventing monotonic improvement. Ideally, one would like the total error to decrease monotonically with $n$. This motivates the need for additional structural assumptions, one of which is explored in the following section.

The proof of Theorem \ref{thm:error-lower} is subtle and relies on the careful construction of both the Hamiltonian $\operatorname{H}$ and an $\varepsilon$-approximate solver for the true evolution operator $\operatorname{F}$. The main challenge lies in constructing a test function $\psi$ that simultaneously satisfies three key properties: (i) it is a valid wave function with unit $L^2$-norm, (ii) it has bounded Sobolev norm, and (iii) it is sufficiently challenging that our estimator constructed from the training data incurs a large prediction error when applied to $\psi$.

\subsection{Refined Upper Bound Under Stronger Assumptions on PDE Solver}
From the proof of the lower bound in Appendix \ref{appdx:error-lower}, it is evident that establishing Theorem \ref{thm:error-lower} relies on a somewhat unnatural and almost adversarial PDE solver. This raises the question of whether a tighter bound can be achieved for more realistic, non-adversarial PDE solvers. A natural way to model such solvers is to assume that their errors behave as uncorrelated random noise.  

\begin{assumption}\label{assumption:random}  
For the Fourier modes $ \varphi_k $, assume that the \emph{PDE} solver satisfies  $
\operatorname{P}(\varphi_k) = \operatorname{F}(\varphi_k) + \delta_k,$
where $ \delta_k $ is a random variable in $ L^2(\torus^d) $ such that:  (i) $ \mathbb{E}[\|\delta_k\|_{L^2}^2] \leq \varepsilon^2 $  and (ii)  $\,  \mathbb{E}[\langle \delta_k, \delta_{\ell} \rangle_{L^2}] = 0 $ for all $ k \neq \ell $.

\end{assumption}  

This assumption effectively places the problem as the fixed-design regression under uncorrelated, homoscedastic noise-- a well-studied model in statistics. However, unlike in classical statistics, where the learner is given a fixed design matrix, our setting allows the learner to actively choose the design matrix. Under this stronger assumption on the solver, we establish the following improved upper bound on the estimator.
\begin{theorem}[Improved Upper Bound]\label{thm:improved}  
Under Assumption \ref{assumption:random}, the estimator defined in Equation \eqref{eqn:estimator} satisfies  
\[
     \mathbb{E} \left[ \sup_{\|\psi\|_{\Hcal^s} \leq c} \|\widehat{\operatorname{F}}_n(\psi) - \operatorname{F}(\psi)\|_{L^2} \right] \leq \varepsilon + 3^s \, c\, \,  n^{-\frac{s}{d}}.
\]
\end{theorem}

Here, the expectation is taken over the randomness introduced by the $ \delta_k $'s in the estimator. The key advantage of this assumption is that the solver errors do not accumulate across modes. In the worst-case setting of Theorem \ref{thm:error}, the errors can be chosen adversarially to align across the Fourier coefficients of $\psi$. This effectively requires control of sums of absolute values and leading to the growth factor $\gamma_n$. In contrast, under the uncorrelated noise assumption, the cross terms vanish in expectation, so the error is only governed by sums of squared magnitudes. Since the squared magnitudes of Fourier coefficients of $L^2$ functions are always summable, the contribution of solver error is always bounded by a constant.

Notably, the expectation is applied \emph{after} the supremum over all wave functions in the Sobolev ball of radius $ c $. This ensures that we are still bounding the error uniformly rather than in the mean-squared sense as is common in statistical learning theory. The expectation is required only because the uniform error is now a random variable. We defer the proof of Theorem \ref{thm:improved} to Appendix \ref{appdx:thm-improved}.  Lastly, we want to point out that a straightforward adaptation of the proof of Theorem \ref{thm:improved} can improve this result to a high-probability bound, provided the tails of $ \delta_k $'s decay sufficiently fast. In particular, assuming that $ \delta_k $'s are subgaussian in $ L^2(\torus^d) $ is sufficient for this improvement.

\section{Time Generalization}\label{sec:time-gen}

In this section, we only consider the case where the Hamiltonian remains constant over time. Recall that the operator $ \widehat{\operatorname{F}}_n $ is trained to predict the wave function at time $ t = T $. We now analyze the error when using it to evolve the initial wave function over multiple time steps, i.e., at $ t = T, 2T, 3T, \ldots $. For any $ q \in \mathbb{N} $, the true wave function at time $ t = qT $ is given by $\psi_{qT}  = \exp\left( - \imaginary/\hbar\cdot qT \, \operatorname{H}\right)\psi(\cdot, 0) = \operatorname{F}^q \left(\psi(\cdot, 0) \right). $
Here, we use the fact that $\operatorname{H}(s)=\operatorname{H}$ for all $s \geq 0$ and $\operatorname{F} = \exp\left( - \frac{\imaginary}{\hbar} T \, \operatorname{H}\right)$. Now,
our goal is to quantify the deviation of the estimated evolution $ \widehat{\operatorname{F}}_n^q(\psi) $ from the exact solution $ \operatorname{F}^q(\psi) $. To the best of our knowledge, this is the first result that provides explicit \emph{theoretical guarantees} for time generalization in operator learning.

\begin{theorem}\label{thm:time-gen}
 Suppose $\operatorname{P}$ satisfies $ \langle \operatorname{P}(u), \operatorname{P}(v) \rangle_{L^2} = \langle u, v \rangle_{L^2} $. Let $ \gamma_n \lesssim 1 $ when $ 2s > d $, $ \gamma_n \lesssim \sqrt{\log n} $ when $ 2s = d $, and $ \gamma_n \lesssim n^{\frac{1}{2}-\frac{s}{d}} $ when $ 2s < d $. Then, for any $ \psi \in \mathcal{H}^s(\mathbb{T}^d) $, the estimator \eqref{eqn:estimator} satisfies  
 \[ \|\widehat{\operatorname{F}}_n^q(\psi)- \operatorname{F}^q (\psi)\|_{L^2} \leq \left( \varepsilon \gamma_n + 3^s \, n^{-\frac{s}{d}} \right)\,  \sum_{j=0}^{q-1} \|\operatorname{F}^j (\psi)\|_{\Hcal^s}.\]
\end{theorem}

The proof relies on a telescoping decomposition of the multi-step error. Specifically, the difference $\widehat{\operatorname{F}}_n^q(\psi)- \operatorname{F}^q (\psi)$ is written as a sum of one-step errors evaluated along the true trajectory. That is,
\[
\widehat{\operatorname{F}}_n^q(\psi)- \operatorname{F}^q (\psi)
= \sum_{j=0}^{q-1} \widehat{\operatorname{F}}_n^{\,q-1-j} \Big( (\widehat{\operatorname{F}}_n - \operatorname{F}) \operatorname{F}^j (\psi) \Big).
\]
The linearity of the estimator allows this decomposition, while the contraction property of $\widehat{\operatorname{F}}_n$ established in Proposition \ref{prop:unitarity} ensures that applying $\widehat{\operatorname{F}}_n^{\,q-1-j}$ does not amplify the error. As a result, the multi-step error is controlled by a sum of one-step errors evaluated at $\operatorname{F}^j(\psi)$, which leads to the bound in Theorem \ref{thm:time-gen}. The full proof is deferred to Appendix \ref{appdx:time-gen}.

This result establishes that the estimator $ \widehat{\operatorname{F}}_n $ can be used to evolve the wave function beyond the time range covered in the training set. However, this requires the true evolution operator to be sufficiently regular. Specifically, the estimator can evolve the wave function for $ q-1 $ additional steps as long as $ \operatorname{F}^j(\psi) \in \Hcal^s(\mathbb{T}^d) $ for all $ j \leq q-1 $. This regularity requirement is natural, given that Theorem \ref{thm:error} already requires that the input belong to $ \Hcal^s(\mathbb{T}^d) $.  However, instead of requiring $ \widehat{\operatorname{F}}_n^j(\psi) $ to belong to $ \Hcal^s(\torus^d) $, it is sufficient for $ \operatorname{F}^j(\psi) $ to be in $ \Hcal^s(\torus^d) $. 
We again note that the linearity of the estimator $ \widehat{\operatorname{F}}_n $ plays a crucial role in the proof of Theorem \ref{thm:time-gen}. Equally important is the fact that $ \widehat{\operatorname{F}}_n $ is a contraction in $ L^2 $ (property (ii) of Proposition \ref{prop:unitarity}), which ensures that the time generalization bound does not suffer from worse convergence rates in terms of sample size or an exponential dependence on $q$. Thus, it is unclear whether similar guarantees could be derived for generic neural network-based surrogates.   In fact, our empirical results suggest otherwise: our estimator exhibits significantly smaller time extrapolation error at step $j = 16$ than the single-step prediction error ($j = 1$) observed for neural operator surrogates (see Tables \ref{table:exp_results} and \ref{table:time-gen}).


Although Theorem \ref{thm:time-gen} provides a time generalization bound, it is expressed in terms of the rather abstract quantity $ \|\operatorname{F}^j (\psi)\|_{\Hcal^s} $. Naturally, one may ask under what conditions this norm remains bounded. The following result bounds $ \|\operatorname{F}^j(\psi)\|_{\Hcal^s} $ in terms of the properties of the potential function.
\begin{corollary}\label{cor:smooth-potential}
  Suppose the $\varepsilon$-approximate \emph{PDE} solver satisfies $ \langle \operatorname{P}(u), \operatorname{P}(v) \rangle_{L^2} = \langle u, v \rangle_{L^2} $. Let $ \gamma_n \lesssim 1 $ when $ 2s > d $, $ \gamma_n \lesssim \sqrt{\log n} $ when $ 2s = d $, and $ \gamma_n \lesssim n^{\frac{1}{2}-\frac{s}{d}} $ when $ 2s < d $. Then, we have:
  \begin{itemize}
      \item[\emph{(i)}] If $V(x)=a\in \reals $, then
      
\[  \|\widehat{\operatorname{F}}_n^q(\psi)- \operatorname{F}^q (\psi)\|_{L^2} \leq \norm{\psi}_{\Hcal^s} \left( \varepsilon \gamma_n + 3^s \, n^{-\frac{s}{d}} \right) q.\]
\item[\emph(ii)] If $V \in C^{\infty}(\torus^d)$ is real-valued, then $\exists c>0$ such that
\[\|\widehat{\operatorname{F}}_n^q(\psi)- \operatorname{F}^q (\psi)\|_{L^2} \leq \norm{\psi}_{\Hcal^s} \left( \varepsilon \gamma_n + 3^s \, n^{-\frac{s}{d}} \right) \cdot c \, q(1+T(q-1)). \]
\item[\emph{(iii)}] If $V \in \Hcal^r(\torus^d)$ for $r\geq \max\{s, d/2\}$, then $\exists c>0$ such that
\[    \|\widehat{\operatorname{F}}_n^q(\psi)- \operatorname{F}^q (\psi)\|_{L^2} \leq \norm{\psi}_{\Hcal^s} \left( \varepsilon \gamma_n + 3^s \, n^{-\frac{s}{d}} \right) \cdot \frac{\exp\left( c \, \norm{V}_{\Hcal^r} \cdot q T \right)-1}{\exp\left( c \, \norm{V}_{\Hcal^r} \cdot  T \right)-1}.\]
  \end{itemize}
\end{corollary}

Corollary \ref{cor:smooth-potential} shows that time generalization is possible when the potential function $ V $ is sufficiently smooth. The extrapolation penalty varies with the regularity of $ V $: it is linear in $ q $ when $ V $ is constant, grows polynomially when $ V $ is infinitely differentiable, and becomes exponential in $q$ when $ V $ has only limited Sobolev regularity. Recall that $C^{\infty}({\torus^d})$ are functions for which derivatives of all orders exist and are continuous. Notably, the convergence rate with respect to the sample size $ n $ remains unaffected. Our empirical findings suggest that the exponential dependence on $ q $ may be conservative, and improving this dependence on $q$ remains an open question. 
Note that the dependence on $ V $'s smoothness arises because, when $ V $ lacks regularity, the evolved wave function $ \operatorname{F}^j(\psi) $ may lose Sobolev smoothness. Since $ \widehat{\operatorname{F}}_n $ is recursively applied at each step, smoothness is needed to repeatedly invoke the one-step generalization guarantee. Without smoothness of $V$, time generalization would require a different estimator with small prediction error uniformly over all of $L^2(\torus^d)$. However, such a stronger guarantee for one-step prediction likely requires a different structural assumption on $V$. This phenomenon is also reflected in our empirical results: potentials with discontinuities exhibit noticeably worse time generalization performance compared to smooth potentials (see Tables \ref{table:potentials} and \ref{table:time-gen}).

The proof of part (i) relies on the fact that $\varphi_k$'s are eigenfunctions of the Hamiltonian when $V$ is constant. Part (ii) uses the seminal result due to \cite{bourgain1999growth}. The dependence on $ q $ can be sharpened to $ q(1 + Tq)^{\delta} $ for any $ \delta > 0 $ using the refined estimate from \cite{bourgain1999growth}. For part (iii), we were unable to find a corresponding result in the literature for potentials $ V $ with limited Sobolev regularity (i.e., not $ C^\infty $). We therefore derive this estimate ourselves, adapting techniques from \cite{bourgain1999growth}. See Appendix \ref{appdx:smooth-potential} for the proof of Corollary \ref{cor:smooth-potential}.

\section{Experiments}\label{sec:experiments}

In this section, we empirically evaluate the proposed estimator and compare it with standard neural operator architectures. Our goal is to assess whether a simple linear estimator that explicitly incorporates the structure of the Schr\"odinger equation can outperform more complex nonlinear models that are trained in a purely data-driven manner. In particular, we investigate whether incorporating linearity and unitarity leads to improvements in both sample efficiency and predictive accuracy.

In addition to accuracy, we highlight the computational advantages of our approach. Unlike neural operator models, which require training large parametric networks via stochastic optimization and typically rely on GPU resources, our estimator is constructed directly from a set of solver queries and does not involve any training procedure. As a result, it incurs substantially lower computational overhead and can be implemented efficiently without specialized hardware. Moreover, incorporating additional samples  only incremental updates to the estimator, in contrast to neural operator models that generally require full retraining.

\subsection{Setup}\label{sec:exp_setup}
We now compare the generalization error of our proposed estimator to that of the Fourier Neural Operator (FNO) \citep{li2020fourier}, U-Net Neural Operator (UNO) \citep{rahman2022u}, and DeepONet \citep{lu2019deeponet} surrogate models across several Hamiltonians of practical interest. We note that, while the time-evolution for a time-invariant Hamiltonian can theoretically be computed using a numerical solver for the time-independent Schr\"{o}dinger equation to obtain eigenvalue-eigenfunction pairs $\{(E_{k}, \phi_{k})\}_{|k|_{\infty}\le K_{n}}$ (where we use $\phi$ in place of $\varphi$ to explicitly note that $\phi$ need not be the Fourier modes) of Hamiltonian, fitting for any future queried initial condition $\psi(\cdot,0)$ the coefficients $\{\alpha_k\}$ such that $\psi(\cdot,0) = \sum_{_{|k|_{\infty}\le K_{n}}} \alpha_k \phi_{k}(\cdot)$, and finally estimating $\psi(\cdot,T) = \sum_{_{|k|_{\infty}\le K_{n}}} e^{-i E_{k} T / \hbar} \alpha_k \phi_{k}(\cdot)$, doing so is not feasible in all but the most trivial of Hamiltonians \citep{van2007accurate,leforestier1991comparison}. 
In turn, learning a solution operator is of interest in both cases of time-invariant and time-dependent Hamiltonians, as we consider below.

In each of the experiments, we employed a standard second-order split-step pseudospectral method as the numerical solver $\operatorname{P}$, where fields were solved in natural units, such that $\hbar = 1$ and $m=1$ \citep{weideman1986split} . Experiments were conducted over $\mathbb{T}^{2}$ with a uniform discretization of $256\times 256$, with the exception of the Coulomb and dipole potentials, where solutions were sought over $\mathcal{S}^{2}$, with $(\phi, \theta)\in[0,2\pi)\times[0,\pi]$ discretized in an equiangular grid of size $64\times 32$. $K_n = 16$ was fixed across experiments, meaning the proposed estimator was fitted on $\mathcal{D} := \{(\varphi_{k}, \operatorname{P}(\varphi_{k})\}$ for $k\in \{-16, \ldots, 0, \ldots,  16\}^{2}$, giving a total of $(2 K_n + 1)^{2}$ samples. For the Coulomb and dipole potentials, the estimator was fit on the spherical harmonics basis elements $Y^{m}_{\ell}$ for $\ell=0,...,L_{\max}$ and $m=-\ell,...\ell$, where $L_{\max} = 10$, giving a total of $(L_{\max} + 1)^{2}$ samples. For realism, we assumed such data were measured with noise, that is that measurements were instead made on $(\varphi_{k} + \varepsilon_{\mathrm{in}},  \operatorname{P}(\varphi_{k}) + \varepsilon_{\mathrm{out}})$ for $\varepsilon_{\mathrm{in}}, \varepsilon_{\mathrm{out}} := \sigma \cdot (Z_{\Re} + i Z_{\Im})$ with $Z_{\Re},Z_{\Im}\sim\mathcal{N}(0,1)$ and $\sigma$ being the noise scale parameter. We considered various relative noise scales in the experiments presented below. 

\begin{table}
\centering
\caption{\label{table:potentials} Summary of potentials implemented for experiments. Potentials without an explicit $t$ dependency are time-\textit{independent}.  Full descriptions of these potentials as well as values chosen for the free parameters are provided in \Cref{appdx:exp_potentials}.}
\resizebox{.8\textwidth}{!}{%
\begin{tabular}{clc}
\hline
\textbf{Potential Name} & \textbf{Expression} & \textbf{Domain} \\
\hline
Free Particle & $V(x,y) = 0$ & $\mathbb{T}^2$ \\

Barrier & 
$V(x,y) = \begin{cases}
V_0 & \text{at } x = \pi,\; y \notin [\pi\pm w] \\
0 & \text{else}
\end{cases}$ 
& $\mathbb{T}^2$ \\

Harmonic Oscillator & 
$V(x,y) = \frac{1}{2} m \omega^2 [(x - \pi)^2 + (y - \pi)^2]$ 
& $\mathbb{T}^2$ \\

Random Field & 
$V(x,y) \sim \text{GRF}\left(0,\; \alpha (-\Delta + \beta \identity)^{-\gamma} \right)$ 
& $\mathbb{T}^2$ \\

Paul Trap & 
$V(x,y,t) = \left( \frac{U_0 + V_0 \cos(\omega t)}{r_0^2} \right) (x^2 + y^2)$ 
& $\mathbb{T}^2$ \\

Shaken Lattice & 
$V(x,y,t) = V_0 \cos[k (x - A \sin(\omega t))] + V_0 \cos(k y)$ 
& $\mathbb{T}^2$ \\

Gaussian Pulse & 
$V(x,y,t) = V_0 \exp\left(-\frac{(x - x_0)^2}{2\sigma_x^2} - \frac{(y - y_0)^2}{2\sigma_y^2} \right) \sum_{t_0} e^{-\frac{(t - t_0)^2}{2 \sigma_t^2}}$ 
& $\mathbb{T}^2$ \\

Coulomb & $V(\theta,\phi) = -k \frac{e^2}{r^2}$ & $\mathcal{S}^2$ \\

Coulomb Dipole & $V(\theta,\phi) = V_0 \cos(\theta)$ & $\mathcal{S}^2$ \\
\hline
\end{tabular}
}
\end{table}

Initial conditions for test data were drawn from a Gaussian Random Field (GRF) by defining a field $\psi(\cdot, 0) = \sum_{|k|_{\infty}\le N/2} c_{k} \varphi_{k}$, where $N/2$ is the Nyquist frequency and $c_{k} = Z \alpha^{1/2} (4\pi^2 || k ||^2_{2} + \beta)^{-\gamma / 2}$, where $Z\sim\mathcal{N}(0,1)$, $\alpha=1$, $\beta=1$, and $\gamma=4$. These are samples from a Gaussian distribution on $L^2(\torus^d)$ with mean $0$ and covariance operator $\alpha(-\Delta +\beta \identity)^{-\gamma}$. Similar draws were made for the Coulomb and dipole potentials, with the expansion being over $\{Y^{m}_{\ell}\}$ instead of $\{\varphi_{k}\}$.

As Fourier Neural Operators are intended to be trained on data drawn i.i.d. from the test distribution, we generated a separate training dataset $\mathcal{D}' := \{(\psi_{i}(\cdot,0), P(\varphi_{i})\}$ identically to the test points, such that $|\mathcal{D}'| = |\mathcal{D}|$.
FNOs were fitted with $K_{n}$ modes using Adam \citep{kingma2014adam} for 20 epochs, where the complex fields were handled in the standard manner of representing the real and imaginary components as separate channels as in \citep{mizera2023scattering}. This setup was identically repeated for the UNO and DeepONet. We also attempted to train these neural operators using the basis functions used to construct our linear estimator instead of i.i.d. samples. However, this resulted in all models collapsing to predicting near-zero fields on the test set. As there is currently no established active data-collection baseline for neural operators, and to ensure a fair comparison, we therefore only report results for FNO, UNO, and DeepONet trained on i.i.d. samples drawn from the test distribution. We note that the observed improvements arise from a combination of two factors: the use of actively selected queries and the incorporation of problem-specific structure through a linear, unitary-aware estimator. Isolating the individual contributions of these components is an important direction for future work.

For Coulomb and dipole potentials, we used the Spherical FNO proposed by \cite{bonev2023spherical}. Since no such extension to spherical domains exists, we exclude DeepOnet from this comparison.

\subsection{Estimator Accuracy}\label{sec:est_acc}

We consider several Hamiltonians of interest from quantum mechanics, drawing examples from both classical settings and of recent research interest, summarized in \Cref{table:potentials} with full descriptions deferred to \Cref{appdx:exp_potentials}.

\begin{table}
\caption{\label{table:exp_results} Average relative errors across different Hamiltonians for a relative noise level of 0.1\%, computed over 100 i.i.d. test samples, with standard deviations in parentheses. 
Note that, for the Coulomb and dipole potential, the FNO columns instead refer to SFNO models. Dashes for DeepONet and UNO indicate that they do not handle functions on a spherical domain.
}
\centering
\resizebox{\textwidth}{!}{%
\begin{tabular}{ccccc}
\toprule
{} &                    FNO &                    UNO &               DeepONet &                          Linear \\
\midrule
Barrier             &  5.146e-02 (1.897e-02) &   2.79e-02 (7.088e-03) &  1.733e-01 (6.926e-02) &  \textbf{1.596e-03 (1.584e-05)} \\
Coulomb             &  5.173e-02 (1.733e-02) &                    --- &                    --- &  \textbf{1.464e-03 (1.437e-05)} \\
Dipole              &  5.516e-02 (1.149e-02) &                    --- &                    --- &  \textbf{1.462e-03 (1.906e-05)} \\
Free                &   1.65e-02 (1.094e-02) &  1.398e-02 (8.162e-03) &  1.582e-01 (8.435e-02) &  \textbf{1.595e-03 (1.673e-05)} \\
Gaussian Pulse      &  5.448e-02 (2.072e-02) &  9.535e-02 (8.224e-03) &  2.055e-01 (6.288e-02) &  \textbf{1.597e-03 (1.495e-05)} \\
Harmonic Oscillator &  4.249e-02 (2.163e-02) &  1.005e-01 (2.328e-02) &  1.605e-01 (9.755e-02) &  \textbf{1.598e-03 (1.845e-05)} \\
Paul Trap           &  1.179e-01 (4.435e-02) &  9.955e-01 (1.055e-02) &  7.573e-01 (6.236e-02) &  \textbf{1.597e-03 (1.345e-05)} \\
Random              &  1.738e-02 (9.927e-03) &  1.652e-02 (7.273e-03) &  1.655e-01 (1.102e-01) &  \textbf{1.594e-03 (1.659e-05)} \\
Shaken Lattice      &  7.918e-02 (2.003e-02) &  2.093e-02 (1.143e-02) &  2.032e-01 (1.080e-01) &  \textbf{1.595e-03 (2.154e-05)} \\
\bottomrule
\end{tabular}
}
\end{table}

As alluded to earlier, we consider various relative noise levels, sweeping over relative noises of 0.01\% to 1\%. We present the results for a relative noise level of 0.1\% in \Cref{table:exp_results} and defer the results over the remaining noise levels to \Cref{appdx:exp_noise_levels}. From these results, we see that the proposed estimator significantly outperforms alternative operator learning methods across all the Hamiltonians, both time-independent and time-dependent, and over both the Fourier and spherical harmonics bases, by leveraging the known linear structure of the true solution operator. Notably, as discussed in the experimental setup, the test samples were drawn over the full spectrum, i.e., with modes defined up to the Nyquist frequency. So, the test samples can be outside the span of modes used to define the estimator. If, however, such test points are restricted to be in the span of the basis elements used to define the estimator, we observe the perfect recovery; such results are provided in \Cref{appdx:trunc-exp}. 

\subsection{Estimator Under Partial Observation}\label{sec:exp_par_obs}
We now test for the robustness of the estimator to partial observation. Analogous to the noisy observation model described in \Cref{sec:exp_setup}, many practical settings involve only partial observation of the evolved state, for which reason we sought to characterize the relative robustness of the estimators in such a setup. To simulate partial observability, we assume the spectrum of the evolved state $\operatorname{P}(\psi_{i})$ has a random fraction of its modes zeroed out at training time. We drop these uniformly at random with a fixed probability across the Fourier modes for rectilinear potentials and similarly for the spherical harmonics coefficients for spherical potentials. The estimators were then fitted against this masked dataset and evaluated against a full, unmasked dataset, i.e. against a test dataset equivalent to that used in \Cref{sec:est_acc}. The identical procedure was used to generate the data for the FNO and DeepONet models. We fixed the noise level to be at a relative level of 0.1\% and again compared the performances using the relative errors.

The results for a mask probability of 10\% are given in \Cref{table:par_obs_results} and those for 20\% deferred to \Cref{appdx:addn_par_obs}. We again see that the linear estimator robustly handles such partial measurement better than the alternative estimators considered. 

\begin{table}
\caption{\label{table:par_obs_results} Average relative errors across different Hamiltonians for a masking probability of 10\%, computed over 100 i.i.d. test samples, with standard deviations in parentheses. 
Note that, for the Coulomb and dipole potential, the FNO columns instead refer to SFNO models. Dashes for DeepONet and UNO indicate that they do not handle functions on a spherical domain.
}
\centering
\resizebox{\textwidth}{!}{%
\begin{tabular}{ccccc}
\toprule
{} &                    FNO &                    UNO &               DeepONet &                          Linear \\
\midrule
Barrier             &  1.616e-01 (2.333e-02) &  2.746e-01 (7.726e-02) &  3.258e-01 (9.531e-02) &  \textbf{1.594e-03 (1.453e-05)} \\
Coulomb             &  2.200e-01 (4.145e-02) &                    --- &                    --- &  \textbf{1.461e-03 (1.792e-05)} \\
Dipole              &   1.602e-01 (3.89e-02) &                    --- &                    --- &  \textbf{1.463e-03 (1.597e-05)} \\
Free                &  8.848e-02 (5.676e-02) &   1.377e-01 (3.91e-02) &  3.775e-01 (1.204e-01) &  \textbf{1.595e-03 (1.818e-05)} \\
Gaussian Pulse      &   1.219e-01 (7.07e-02) &  1.879e-01 (3.594e-02) &  3.021e-01 (8.951e-02) &  \textbf{1.597e-03 (1.552e-05)} \\
Harmonic Oscillator &  2.255e-01 (1.123e-01) &  1.319e-01 (3.497e-02) &  2.895e-01 (8.112e-02) &  \textbf{1.598e-03 (1.477e-05)} \\
Paul Trap           &   2.030e-01 (6.04e-02) &  1.607e-01 (3.037e-02) &  4.804e-01 (5.122e-02) &  \textbf{1.595e-03 (1.485e-05)} \\
Random              &  7.758e-02 (3.243e-02) &   9.36e-02 (1.512e-02) &  2.690e-01 (1.278e-01) &  \textbf{9.358e-03 (6.039e-04)} \\
Shaken Lattice      &  3.179e-01 (6.896e-02) &  1.636e-01 (1.917e-02) &  2.863e-01 (1.021e-01) &  \textbf{1.595e-03 (1.537e-05)} \\
\bottomrule
\end{tabular}
}
\end{table}

\subsection{Time Generalization}\label{sec:timegen-experiments}

To assess time generalization of our estimator, we start with an initial wave $ \psi(\cdot,0) $ and evolve it iteratively using both the true flow and our learned operator. Specifically, for $j = 1, \dots, q $, the true evolution gives $
\mathrm{P}^j(\psi(\cdot,0)),$
and we generate noisy data by adding noise as described in ~\Cref{sec:exp_setup}, that is $
\mathrm{P}^j(\psi(\cdot,0)) + \varepsilon.$
 In parallel, starting from the noisy version of the initial condition $ \psi(\cdot,0)+\varepsilon $, we evolve our estimator to obtain $ \widehat{\mathrm{F}}_n^j(\psi(\cdot,0)+\varepsilon) $.  At each time step $ j = 1, \dots, q $, we compute the relative error between the two.
The test initial conditions are sampled from the GRF prior described earlier. Table~\ref{table:time-gen} shows the average relative errors for a relative noise level of 0.1\%, evaluated over 100 i.i.d. test samples. Results for a higher noise level of 1\% are deferred to Appendix~\ref{appdx:exp_timegen}.

For the free, harmonic oscillator, and random potential, the error remains nearly constant across time steps, indicating long-term generalization. A similar trend is observed for the Coulomb and dipole potentials on the sphere. In contrast, the error increases sharply, by an order of magnitude at $j = 2$ for the barrier potential, which is likely due to its discontinuity. For time-dependent potentials such as the Paul trap and Gaussian pulse, the estimator incurs larger errors at later steps. Note that our time-generalization bounds do not apply to these time-varying Hamiltonians.
Overall, the results show that our estimator generalizes well beyond the training time points for sufficiently smooth potentials. Furthermore, the empirical error growth is notably slower than the exponential bound suggested in part (iii) of Corollary \ref{cor:smooth-potential}. We leave this possible refinement for future work.

\begin{table}[t]
\caption{\label{table:time-gen} Average relative time-generalization errors across different Hamiltonians for a relative noise level of 0.1\%, computed over 100 i.i.d.\ test samples, with standard deviations shown in parentheses.}
\centering
\resizebox{\textwidth}{!}{%
\begin{tabular}{cccccc}
\toprule
Hamiltonian & $j=1$ & $j=2$ & $j=4$ & $j=8$ & $j=16$ \\
\midrule
Barrier & 1.592e-03 (1.190e-05) & 1.805e-02 (3.739e-03) & 1.823e-02 (3.964e-03) & 1.692e-02 (3.349e-03) & 1.641e-02 (3.363e-03) \\
Coulomb & 1.465e-03 (1.748e-05) & 1.468e-03 (1.438e-05) & 1.462e-03 (1.374e-05) & 1.465e-03 (1.627e-05) & 1.461e-03 (1.715e-05) \\
Dipole & 1.462e-03 (1.731e-05) & 1.467e-03 (1.768e-05) & 1.463e-03 (1.856e-05) & 1.460e-03 (1.661e-05) & 1.469e-03 (1.811e-05) \\
Free & 1.591e-03 (1.221e-05) & 1.591e-03 (1.383e-05) & 1.593e-03 (1.241e-05) & 1.588e-03 (1.274e-05) & 1.591e-03 (1.385e-05) \\
Gaussian Pulse & 1.592e-03 (1.232e-05) & 1.546e-02 (5.556e-03) & 1.799e-02 (6.853e-03) & 1.852e-02 (6.885e-03) & 1.988e-02 (7.362e-03) \\
Harmonic Oscillator & 1.590e-03 (1.370e-05) & 1.721e-03 (3.145e-05) & 1.659e-03 (2.303e-05) & 1.666e-03 (2.477e-05) & 1.712e-03 (3.330e-05) \\
Paul Trap & 1.591e-03 (1.499e-05) & 1.511e-01 (2.023e-02) & 4.536e-01 (4.893e-02) & 6.344e-01 (4.913e-02) & 6.670e-01 (4.594e-02) \\
Random Lattice & 1.591e-03 (1.269e-05) & 1.592e-03 (1.350e-05) & 1.591e-03 (1.223e-05) & 1.589e-03 (1.128e-05) & 1.589e-03 (1.306e-05) \\
Shaken Lattice & 1.592e-03 (1.382e-05) & 5.507e-03 (6.975e-04) & 5.516e-03 (6.573e-04) & 3.740e-03 (2.924e-04) & 6.269e-03 (5.130e-04) \\
\bottomrule
\end{tabular}
}
\end{table}

\section{Discussion}
In this work, we proposed a structure-aware linear surrogate for learning the evolution operator of the time-dependent Schr\"{o}dinger equation using actively chosen spectral queries. Our results demonstrate that exploiting known structural properties of Schr\"{o}dinger dynamics, linearity, spectral structure, and unitarity, can lead to highly accurate operator surrogates with strong theoretical guarantees.

While our method provides rigorous theoretical guarantees, it is currently limited to a single particle system. A natural direction for future work is to extend the method to handle a system with $N$ interacting particles. For many-body systems, the wave function must satisfy symmetry constraints: symmetric for bosons and antisymmetric for fermions. A promising direction is therefore to adapt the spectral probing strategy to basis functions that respect these symmetry structures. However, extending our approach to systems of $N$ interacting particles presents significant challenges due to the exponential growth of the configuration space.

Additionally, our method estimates the evolution operator for a fixed Hamiltonian. An interesting extension would be to develop a surrogate that takes both the initial wave and the potential as inputs and predicts the wave function at time $ T $.  Such a method would allow generalization across different system configurations, which can be used in applications such as qubit design. However, this requires moving beyond linear operators, as the ground truth operator mapping from $ (\psi(\cdot, 0), V) $ to $ \psi(\cdot, T) $ is nonlinear. Neural network-based operator learners could naturally represent such nonlinear mappings, although providing rigorous theoretical guarantees for these models remains a significant challenge.

Finally, our work highlights the potential of active experimental design for operator learning. In this paper we employed spectral probing based on Fourier modes, but exploring more sophisticated active querying strategies may further improve data efficiency. Understanding how structural properties of dynamical systems interact with operator learning and experimental design remains an exciting direction for future research.


\bibliographystyle{plainnat}
\bibliography{references}

@book{Major2005,
  author    = {Fouad G. Major and Viorica N. Gheorghe and Günther Werth},
  title     = {Charged Particle Traps: Physics and Techniques of Charged Particle Field Confinement},
  year      = {2005},
  publisher = {Springer},
  isbn      = {978-3-540-40710-5},
  doi       = {10.1007/b137836},
  url       = {https://link.springer.com/book/10.1007/b137836}
}

@article{bernardini2023quantum,
  title={Quantum computing with trapped ions: a beginner’s guide},
  author={Bernardini, Francesco and Chakraborty, Abhijit and Ord{\'o}{\~n}ez, Carlos R},
  journal={European Journal of Physics},
  volume={45},
  number={1},
  pages={013001},
  year={2023},
  publisher={IOP Publishing}
}

@book{grafakos2008classical,
  title={Classical fourier analysis},
  author={Grafakos, Loukas and others},
  volume={2},
  year={2008},
  publisher={Springer}
}

@book{taylorpartial,
  title={Partial Differential Equations {I} Basic Theory},
  author={Taylor, Michael E},
  publisher={Springer},
  year={2011}
}

@book{lord2014introduction,
  title={An introduction to computational stochastic PDEs},
  author={Lord, Gabriel J and Powell, Catherine E and Shardlow, Tony},
  volume={50},
  year={2014},
  publisher={Cambridge University Press}
}

@article{kovachki2023neural,
  title={Neural operator: Learning maps between function spaces with applications to pdes},
  author={Kovachki, Nikola and Li, Zongyi and Liu, Burigede and Azizzadenesheli, Kamyar and Bhattacharya, Kaushik and Stuart, Andrew and Anandkumar, Anima},
  journal={Journal of Machine Learning Research},
  volume={24},
  number={89},
  pages={1--97},
  year={2023}
}

@article{li2020fourier,
  title={Fourier neural operator for parametric partial differential equations},
  author={Li, Zongyi and Kovachki, Nikola and Azizzadenesheli, Kamyar and Liu, Burigede and Bhattacharya, Kaushik and Stuart, Andrew and Anandkumar, Anima},
  journal={International Conference on Learning Representations},
  year={2021}
}

@article{bhattacharya2021model,
  title={Model reduction and neural networks for parametric PDEs},
  author={Bhattacharya, Kaushik and Hosseini, Bamdad and Kovachki, Nikola B and Stuart, Andrew M},
  journal={The SMAI journal of computational mathematics},
  volume={7},
  pages={121--157},
  year={2021}
}

@book{hsing2015theoretical,
  title={Theoretical foundations of functional data analysis, with an introduction to linear operators},
  author={Hsing, Tailen and Eubank, Randall},
  volume={997},
  year={2015},
  publisher={John Wiley \& Sons}
}

@article{subedi2024benefits,
  title={On the Benefits of Active Data Collection in Operator Learning},
  author={Subedi, Unique and Tewari, Ambuj},
  journal={International Conference on Machine Learning (ICML)},
  year={2025}
}

@article{lu2021learning,
  title={Learning nonlinear operators via DeepONet based on the universal approximation theorem of operators},
  author={Lu, Lu and Jin, Pengzhan and Pang, Guofei and Zhang, Zhongqiang and Karniadakis, George Em},
  journal={Nature machine intelligence},
  volume={3},
  number={3},
  pages={218--229},
  year={2021},
  publisher={Nature Publishing Group UK London}
}

@article{shah2024fourier,
  title={Fourier neural operators for learning dynamics in quantum spin systems},
  author={Shah, Freya and Patti, Taylor L and Berner, Julius and Tolooshams, Bahareh and Kossaifi, Jean and Anandkumar, Anima},
  journal={arXiv preprint arXiv:2409.03302},
  year={2024}
}

@article{zhang2024artificial,
  title={Artificial-intelligence-based surrogate solution of dissipative quantum dynamics: physics-informed reconstruction of the universal propagator},
  author={Zhang, Jiaji and Benavides-Riveros, Carlos L and Chen, Lipeng},
  journal={The Journal of Physical Chemistry Letters},
  volume={15},
  number={13},
  pages={3603--3610},
  year={2024},
  publisher={ACS Publications}
}

@article{zhang2025neural,
  title={Neural quantum propagators for driven-dissipative quantum dynamics},
  author={Zhang, Jiaji and Benavides-Riveros, Carlos L and Chen, Lipeng},
  journal={Physical Review Research},
  volume={7},
  number={1},
  pages={L012013},
  year={2025},
  publisher={APS}
}

@article{mizera2023scattering,
  title={Scattering with neural operators},
  author={Mizera, Sebastian},
  journal={Physical Review D},
  volume={108},
  number={10},
  pages={L101701},
  year={2023},
  publisher={APS}
}

@article{niarchos2024learning,
  title={Learning S-matrix phases with neural operators},
  author={Niarchos, Vasilis and Papageorgakis, Constantinos},
  journal={Physical Review D},
  volume={110},
  number={4},
  pages={045020},
  year={2024},
  publisher={APS}
}

@book{sakurai2020modern,
  title={Modern quantum mechanics},
  author={Sakurai, Jun John and Napolitano, Jim},
  year={2020},
  publisher={Cambridge University Press}
}

@article{shah2022physics,
  title={Physics-Informed Neural Networks as Solvers for the Time-Dependent Schrodinger Equation},
  author={Shah, Karan and Stiller, Patrick and Hoffmann, Nico and Cangi, Attila},
  journal={arXiv preprint arXiv:2210.12522},
  year={2022}
}

@article{cuomo2022scientific,
  title={Scientific machine learning through physics--informed neural networks: Where we are and what’s next},
  author={Cuomo, Salvatore and Di Cola, Vincenzo Schiano and Giampaolo, Fabio and Rozza, Gianluigi and Raissi, Maziar and Piccialli, Francesco},
  journal={Journal of Scientific Computing},
  volume={92},
  number={3},
  pages={88},
  year={2022},
  publisher={Springer}
}

@article{mohammed2024optical,
  title={The optical structures for the fractional chiral nonlinear Schr{\"o}dinger equation with time-dependent coefficients},
  author={Mohammed, Wael W and Iqbal, Naveed and Bourazza, S and Elsayed, Elsayed M},
  journal={Optical and Quantum Electronics},
  volume={56},
  number={9},
  pages={1476},
  year={2024},
  publisher={Springer}
}

@article{chen2024simulating,
  title={Simulating decoherence of coupled two spin qubits using generalized cluster correlation expansion},
  author={Chen, Xiao and Hoffman, Silas and Fry, James N and Cheng, Hai-Ping},
  journal={arXiv preprint arXiv:2402.18722},
  year={2024}
}

@article{peskin1993solution,
  title={The solution of the time-dependent Schr{\"o}dinger equation by the $(t, t')$ method: Theory, computational algorithm and applications},
  author={Peskin, Uri and Moiseyev, Nimrod},
  journal={The Journal of chemical physics},
  volume={99},
  number={6},
  pages={4590--4596},
  year={1993},
  publisher={American Institute of Physics}
}

@article{nagele2023decoherence,
  title={Decoherence: a numerical study},
  author={Nagele, Chris and Janssen, Oliver and Kleban, Matthew},
  journal={Journal of Physics A: Mathematical and Theoretical},
  volume={56},
  number={8},
  pages={085301},
  year={2023},
  publisher={IOP Publishing}
}

@article{wu2014rabi,
  title={Rabi oscillations, decoherence, and disentanglement in a qubit--spin-bath system},
  author={Wu, Ning and Nanduri, Arun and Rabitz, Herschel},
  journal={Physical Review A},
  volume={89},
  number={6},
  pages={062105},
  year={2014},
  publisher={APS}
}

@article{shahrokhi2007airfoil,
  title={Airfoil shape parameterization for optimum Navier--Stokes design with genetic algorithm},
  author={Shahrokhi, Ava and Jahangirian, Alireza},
  journal={Aerospace science and technology},
  volume={11},
  number={6},
  pages={443--450},
  year={2007},
  publisher={Elsevier}
}

@article{eyi1994airfoil,
  title={Airfoil design optimization using the Navier-Stokes equations},
  author={Eyi, S{\.I}NAN and Hager, JO and Lee, KD},
  journal={Journal of Optimization Theory and Applications},
  volume={83},
  pages={447--461},
  year={1994},
  publisher={Springer}
}

@article{azizzadenesheli2024neural,
  title={Neural operators for accelerating scientific simulations and design},
  author={Azizzadenesheli, Kamyar and Kovachki, Nikola and Li, Zongyi and Liu-Schiaffini, Miguel and Kossaifi, Jean and Anandkumar, Anima},
  journal={Nature Reviews Physics},
  pages={1--9},
  year={2024},
  publisher={Nature Publishing Group UK London}
}

@article{augenstein2023neural,
  title={Neural operator-based surrogate solver for free-form electromagnetic inverse design},
  author={Augenstein, Yannick and Repan, Taavi and Rockstuhl, Carsten},
  journal={ACS Photonics},
  volume={10},
  number={5},
  pages={1547--1557},
  year={2023},
  publisher={ACS Publications}
}

@article{batzner20223,
  title={E (3)-equivariant graph neural networks for data-efficient and accurate interatomic potentials},
  author={Batzner, Simon and Musaelian, Albert and Sun, Lixin and Geiger, Mario and Mailoa, Jonathan P and Kornbluth, Mordechai and Molinari, Nicola and Smidt, Tess E and Kozinsky, Boris},
  journal={Nature communications},
  volume={13},
  number={1},
  pages={2453},
  year={2022},
  publisher={Nature Publishing Group UK London}
}

@article{merchant2023scaling,
  title={Scaling deep learning for materials discovery},
  author={Merchant, Amil and Batzner, Simon and Schoenholz, Samuel S and Aykol, Muratahan and Cheon, Gowoon and Cubuk, Ekin Dogus},
  journal={Nature},
  volume={624},
  number={7990},
  pages={80--85},
  year={2023},
  publisher={Nature Publishing Group UK London}
}

@article{van2007accurate,
  title={Accurate numerical solutions of the time-dependent Schr{\"o}dinger equation},
  author={Van Dijk, W and Toyama, FM},
  journal={Physical Review E—Statistical, Nonlinear, and Soft Matter Physics},
  volume={75},
  number={3},
  pages={036707},
  year={2007},
  publisher={APS}
}

@article{leforestier1991comparison,
  title={A comparison of different propagation schemes for the time dependent Schr{\"o}dinger equation},
  author={Leforestier, Claude and Bisseling, RH and Cerjan, Charly and Feit, MD and Friesner, Rich and Guldberg, A and Hammerich, A and Jolicard, G and Karrlein, W and Meyer, H-D and others},
  journal={Journal of Computational Physics},
  volume={94},
  number={1},
  pages={59--80},
  year={1991},
  publisher={Elsevier}
}

@article{weideman1986split,
  title={Split-step methods for the solution of the nonlinear Schr{\"o}dinger equation},
  author={Weideman, J Andr{\'e} C and Herbst, Ben M},
  journal={SIAM Journal on Numerical Analysis},
  volume={23},
  number={3},
  pages={485--507},
  year={1986},
  publisher={SIAM}
}

@article{kingma2014adam,
  title={Adam: A method for stochastic optimization},
  author={Kingma, Diederik P and Ba, Jimmy},
  journal={arXiv preprint arXiv:1412.6980},
  year={2014}
}

@inproceedings{bonev2023spherical,
  title={Spherical fourier neural operators: Learning stable dynamics on the sphere},
  author={Bonev, Boris and Kurth, Thorsten and Hundt, Christian and Pathak, Jaideep and Baust, Maximilian and Kashinath, Karthik and Anandkumar, Anima},
  booktitle={International conference on machine learning},
  pages={2806--2823},
  year={2023},
  organization={PMLR}
}

@article{behzadan2021multiplication,
  title={Multiplication in Sobolev spaces, revisited},
  author={Behzadan, Ali and Holst, Michael},
  journal={Arkiv f{\"o}r Matematik},
  volume={59},
  number={2},
  pages={275--306},
  year={2021},
  publisher={Lehigh University Bethlehem, Penn., USA}
}

@article{li2024physics,
  title={Physics-informed neural operator for learning partial differential equations},
  author={Li, Zongyi and Zheng, Hongkai and Kovachki, Nikola and Jin, David and Chen, Haoxuan and Liu, Burigede and Azizzadenesheli, Kamyar and Anandkumar, Anima},
  journal={ACM/JMS Journal of Data Science},
  volume={1},
  number={3},
  pages={1--27},
  year={2024},
  publisher={ACM New York, NY}
}

@article{richter2022neural,
  title={Neural conservation laws: A divergence-free perspective},
  author={Richter-Powell, Jack and Lipman, Yaron and Chen, Ricky TQ},
  journal={Advances in Neural Information Processing Systems},
  volume={35},
  pages={38075--38088},
  year={2022}
}

@article{bourgain1999growth,
  title={On growth of Sobolev norms in linear Schr{\"o}dinger equations with smooth time dependent potential},
  author={Bourgain, Jean},
  journal={Journal d’Analyse Math{\'e}matique},
  volume={77},
  number={1},
  pages={315--348},
  year={1999},
  publisher={Springer}
}

@article{delort2010growth,
  title={Growth of Sobolev norms of solutions of linear Schr{\"o}dinger equations on some compact manifolds},
  author={Delort, Jean-Marc},
  journal={International Mathematics Research Notices},
  volume={2010},
  number={12},
  pages={2305--2328},
  year={2010},
  publisher={OUP}
}

@article{mills2017deep,
  title={Deep learning and the Schr{\"o}dinger equation},
  author={Mills, Kyle and Spanner, Michael and Tamblyn, Isaac},
  journal={Physical Review A},
  volume={96},
  number={4},
  pages={042113},
  year={2017},
  publisher={APS}
}

@article{hermann2023ab,
  title={Ab initio quantum chemistry with neural-network wavefunctions},
  author={Hermann, Jan and Spencer, James and Choo, Kenny and Mezzacapo, Antonio and Foulkes, W Matthew C and Pfau, David and Carleo, Giuseppe and No{\'e}, Frank},
  journal={Nature Reviews Chemistry},
  volume={7},
  number={10},
  pages={692--709},
  year={2023},
  publisher={Nature Publishing Group UK London}
}

@article{nys2024ab,
  title={Ab-initio variational wave functions for the time-dependent many-electron Schr{\"o}dinger equation},
  author={Nys, Jannes and Pescia, Gabriel and Sinibaldi, Alessandro and Carleo, Giuseppe},
  journal={Nature communications},
  volume={15},
  number={1},
  pages={9404},
  year={2024},
  publisher={Nature Publishing Group UK London}
}

@article{stepaniants2023learning,
  title={Learning partial differential equations in reproducing kernel Hilbert spaces},
  author={Stepaniants, George},
  journal={Journal of Machine Learning Research},
  volume={24},
  number={86},
  pages={1--72},
  year={2023}
}

@article{boulle2022data,
  title={Data-driven discovery of Green’s functions with human-understandable deep learning},
  author={Boull{\'e}, Nicolas and Earls, Christopher J and Townsend, Alex},
  journal={Scientific reports},
  volume={12},
  number={1},
  pages={4824},
  year={2022},
  publisher={Nature Publishing Group UK London}
}

@article{bisio2010optimal,
  title={Optimal quantum learning of a unitary transformation},
  author={Bisio, Alessandro and Chiribella, Giulio and D’Ariano, Giacomo Mauro and Facchini, Stefano and Perinotti, Paolo},
  journal={Physical Review A—Atomic, Molecular, and Optical Physics},
  volume={81},
  number={3},
  pages={032324},
  year={2010},
  publisher={APS}
}

@inproceedings{hyland2017learning,
  title={Learning unitary operators with help from u (n)},
  author={Hyland, Stephanie and R{\"a}tsch, Gunnar},
  booktitle={Proceedings of the AAAI Conference on Artificial Intelligence},
  volume={31},
  year={2017}
}

@article{belov2024partially,
  title={Partially unitary learning},
  author={Belov, Mikhail Gennadievich and Malyshkin, Vladislav Gennadievich},
  journal={Physical Review E},
  volume={110},
  number={5},
  pages={055306},
  year={2024},
  publisher={APS}
}

@article{johansson2012qutip,
  title={QuTiP: An open-source Python framework for the dynamics of open quantum systems},
  author={Johansson, J Robert and Nation, Paul D and Nori, Franco},
  journal={Computer physics communications},
  volume={183},
  number={8},
  pages={1760--1772},
  year={2012},
  publisher={Elsevier}
}

@article{huang2021learning,
  title={Learning Unitary Transformation by Quantum Machine Learning Model.},
  author={Huang, Yi-Ming and Li, Xiao-Yu and Zhu, Yi-Xuan and Lei, Hang and Zhu, Qing-Sheng and Yang, Shan},
  journal={Computers, Materials \& Continua},
  volume={68},
  number={1},
  year={2021}
}

@article{de2023convergence,
  title={Convergence rates for learning linear operators from noisy data},
  author={de Hoop, Maarten V and Kovachki, Nikola B and Nelsen, Nicholas H and Stuart, Andrew M},
  journal={SIAM/ASA Journal on Uncertainty Quantification},
  volume={11},
  number={2},
  pages={480--513},
  year={2023},
  publisher={SIAM}
}

@article{mollenhauer2022learning,
  title={Learning linear operators: Infinite-dimensional regression as a well-behaved non-compact inverse problem},
  author={Mollenhauer, Mattes and M{\"u}cke, Nicole and Sullivan, TJ},
  journal={arXiv preprint arXiv:2211.08875},
  year={2022}
}

@article{liu2022algorithm,
  title={Algorithm advances and applications of time-dependent first-principles simulations for ultrafast dynamics},
  author={Liu, Wen-Hao and Wang, Zhi and Chen, Zhang-Hui and Luo, Jun-Wei and Li, Shu-Shen and Wang, Lin-Wang},
  journal={Wiley Interdisciplinary Reviews: Computational Molecular Science},
  volume={12},
  number={3},
  pages={e1577},
  year={2022},
  publisher={Wiley Online Library}
}

@article{li2020real,
  title={Real-time time-dependent electronic structure theory},
  author={Li, Xiaosong and Govind, Niranjan and Isborn, Christine and DePrince III, A Eugene and Lopata, Kenneth},
  journal={Chemical Reviews},
  volume={120},
  number={18},
  pages={9951--9993},
  year={2020},
  publisher={ACS Publications}
}

@article{van2011efficiency,
  title={Efficiency and accuracy of numerical solutions to the time-dependent Schr{\"o}dinger equation},
  author={van Dijk, W and Brown, J and Spyksma, K},
  journal={Physical Review E—Statistical, Nonlinear, and Soft Matter Physics},
  volume={84},
  number={5},
  pages={056703},
  year={2011},
  publisher={APS}
}

@article{zheng2014floquet,
  title={Floquet topological states in shaking optical lattices},
  author={Zheng, Wei and Zhai, Hui},
  journal={Physical Review A},
  volume={89},
  number={6},
  pages={061603},
  year={2014},
  publisher={APS}
}

@article{weidner2017atom,
  title={Atom interferometry using a shaken optical lattice},
  author={Weidner, CA and Yu, Hoon and Kosloff, Ronnie and Anderson, Dana Z},
  journal={Physical Review A},
  volume={95},
  number={4},
  pages={043624},
  year={2017},
  publisher={APS}
}

@article{kiely2016shaken,
  title={Shaken not stirred: creating exotic angular momentum states by shaking an optical lattice},
  author={Kiely, Anthony and Benseny, Albert and Busch, Thomas and Ruschhaupt, Andreas},
  journal={Journal of Physics B: Atomic, Molecular and Optical Physics},
  volume={49},
  number={21},
  pages={215003},
  year={2016},
  publisher={IOP Publishing}
}

@article{deutsch2000quantum,
  title={Quantum computing with neutral atoms in an optical lattice},
  author={Deutsch, Ivan H and Brennen, Gavin K and Jessen, Poul S},
  journal={Fortschritte der Physik: Progress of Physics},
  volume={48},
  number={9-11},
  pages={925--943},
  year={2000},
  publisher={Wiley Online Library}
}

@article{zhang2006manipulation,
  title={Manipulation of single neutral atoms in optical lattices},
  author={Zhang, Chuanwei and Rolston, SL and Das Sarma, S},
  journal={Physical Review A—Atomic, Molecular, and Optical Physics},
  volume={74},
  number={4},
  pages={042316},
  year={2006},
  publisher={APS}
}

@article{rais2022bound,
  title={Bound-state formation in time-dependent potentials},
  author={Rais, Jan and van Hees, Hendrik and Greiner, Carsten},
  journal={Physical Review C},
  volume={106},
  number={6},
  pages={064004},
  year={2022},
  publisher={APS}
}

@article{rahman2022u,
  title={U-no: U-shaped neural operators},
  author={Rahman, Md Ashiqur and Ross, Zachary E and Azizzadenesheli, Kamyar},
  journal={arXiv preprint arXiv:2204.11127},
  year={2022}
}

@article{lu2019deeponet,
  title={Deeponet: Learning nonlinear operators for identifying differential equations based on the universal approximation theorem of operators},
  author={Lu, Lu and Jin, Pengzhan and Karniadakis, George Em},
  journal={arXiv preprint arXiv:1910.03193},
  year={2019}
}

@article{lin1993exact,
  title={Exact diagonalization methods for quantum systems},
  author={Lin, HQ and Gubernatis, JE and Gould, Harvey and Tobochnik, Jan},
  journal={Computers in Physics},
  volume={7},
  number={4},
  pages={400--407},
  year={1993},
  publisher={American Institute of Physics}
}

@incollection{fletcher1984computational,
  title={Computational galerkin methods},
  author={Fletcher, Clive AJ},
  booktitle={Computational galerkin methods},
  pages={72--85},
  year={1984},
  publisher={Springer}
}

\appendix

\section{Extended Related Works}\label{appdx:ext-related-works}

In recent years, there has been considerable work on using machine learning methods to solve the static (time-independent) Schr\"{o}dinger equation for many-body electronic systems. See the review article by \cite{hermann2023ab} for an overview. These methods typically parametrize the ground state wave function $ \psi_{\theta} $ using a neural network and optimize the parameters by minimizing the energy functional $ \langle \psi_{\theta}, \operatorname{H} \psi_{\theta} \rangle_{L^2} $. This framework has also been extended to the time-dependent Schr\"{o}dinger equation for many-electron systems by \cite{nys2024ab}. This line of work is closely related to Physics-Informed Neural Networks (PINNs), which approximate solutions to PDEs by fitting a neural network ansatz that satisfies the variational form of the governing equations; see \cite{shah2022physics} and \cite[Section 3.2.2.3]{cuomo2022scientific}. However, these methods effectively act as solvers, requiring optimization for each new instance, and thus do not amortize computational costs. In contrast, our focus is on learning the global evolution operator directly from data, enabling fast and efficient evaluation for new initial conditions without retraining, thereby significantly reducing downstream computational cost.

An early work in learning solution operators for the Schr\"{o}dinger equation was by \cite{mills2017deep}, who trained a neural network to predict ground-state wave functions from potentials for the time-independent Schr\"{o}dinger equation. More recently, \cite{stepaniants2023learning} proposed an operator learning approach that models the solution operator mapping potentials to ground state wave functions by learning the associated Green's functions in a reproducing kernel Hilbert space (RKHS). A similar strategy was studied by \cite{boulle2022data}, who used rotational neural networks to learn Green's functions for static Schr\"{o}dinger equations. In addition, \cite{boulle2022data} also considered learning the Green's functions associated with time dependent propagator for $1$-dimensional Harmonic oscillator. 

A slightly more general framework was studied by \cite{mizera2023scattering}, who used Fourier Neural Operators (FNOs) \citep{li2020fourier} to estimate the time evolution operator for simple quantum systems, such as random potentials and the double-slit potential. They also studied the ability of the learned operator to generalize across time, extrapolating beyond the training time range. Their learned operator is more flexible than ours in that it takes both the initial wave and the potential function as inputs, rather than assuming a fixed Hamiltonian. Relatedly, \cite{niarchos2024learning} studied learning the phases of amplitudes in scattering problems. Beyond isolated systems, \cite{zhang2024artificial} and \cite{zhang2025neural} extended FNO-based architectures to model dissipative quantum systems that interact with an environment and are possibly driven by external fields, again evaluating time generalization. Most recently, \cite{shah2024fourier} trained FNOs to learn the evolution operator for relatively larger quantum spin systems (up to 8-qubit systems), studying both single-step and multi-step time extrapolation.

\section{Extensions to Non-Periodic Domains}\label{sec:beyond-periodic}
Extending the results from Sections \ref{sec:method} and \ref{sec:time-gen} to general \emph{bounded} domain $ \Omega \subset \mathbb{R}^d $ is straightforward. This requires choosing an orthonormal basis of $ L^2(\Omega) $ and defining a corresponding Sobolev-type space. While any orthonormal basis of $ L^2(\Omega) $ could be used, the most natural choice is the eigenfunctions of the Laplacian, which satisfy the eigenvalue problem,
\[
-\Delta u = \lambda u, \quad \text{subject to appropriate boundary conditions}.
\]
Common boundary conditions include Dirichlet, Neumann, and Robin.  

For the special case $ \Omega = \mathbb{T}^d $, the Laplacian eigenvalues are $ \{4\pi^2 |k|_2^2 \, : \, k \in \mathbb{Z}^d\} $, and the corresponding eigenfunctions are Fourier modes $ \{\varphi_k\}_{k \in \mathbb{Z}^d} $. This motivates defining a more general Sobolev-type space using the eigenpairs of the Laplacian. Let $ \{\lambda_j\}_{j=1}^{\infty} $ denote the eigenvalues such that $ 0 < \lambda_1 \leq \lambda_2 \leq \dots $, and let $ \{\phi_j\}_{j=1}^{\infty} $ be the corresponding eigenfunctions. Then, the Sobolev-type space is defined as  

\[
\mathcal{H}^{s}(\Omega) = \left\{ f \in L^2(\Omega) \, \Big |\, \sum_{j=1}^{\infty} (1+|\lambda_j|)^{s} \, |\langle f, \phi_j \rangle_{L^2}|^2 < \infty \right\}.
\]

For specific choices of $ \Omega $, this space might be defined more naturally using a weight function $ \zeta(\lambda_j)^s $ for some function $ \zeta: (0, \infty) \to (0, \infty) $, or by indexing the eigenvalues with another countable set, such as $ \mathbb{N}^d $ or $ \mathbb{Z}^d $. Nevertheless, this general formulation captures the essential structure of the space. To avoid such indexing issue, one can define this space more implicitly as
\[\mathcal{H}^{s}(\Omega) := \Big\{ f\in L^2(\Omega) \,  |\, (\identity-\Delta)^{s/2} f \in L^2(\Omega) \Big\}.\]
The operator $(\identity-\Delta)^{s/2}$ is called Bessel potential, and this space is also referred to as Bessel potential space.
It is important to note that, unlike in the case of the torus, the equivalence between this Sobolev-type space and the classical Sobolev space defined using differential operators does not generally hold for arbitrary domains $ \Omega $. However, in applied operator learning, sample functions are typically generated using their spectral representation. Thus, we argue that the spectral definition of smoothness is arguably more natural from a practical perspective than the one based on differentials.

To construct the estimator from Section \ref{sec:estimator}, given a sample budget of $ n $, one queries the first $ n $ eigenfunctions $ \phi_1, \phi_2, \dots, \phi_n $ instead of Fourier modes.
The overall proof strategy is expected to extend to general domains, as our proof primarily relies on the use of an orthonormal basis of $L^2$, although additional technical work may be required to verify the necessary spectral properties in general domains. Because the eigenvalues are indexed by $\mathbb{N}$ in $\mathcal{H}^s(\Omega)$, the parameter $s$ in this setting plays a role analogous to $s/d$ in $\mathcal{H}^s(\mathbb{T}^d)$.

In the experiments presented in Section \ref{sec:experiments}, we also consider the case where $ \Omega $ is a sphere and use spherical harmonics, which are the eigenfunctions of the Laplacian on a sphere. Similarly, Bessel functions serve as the Laplacian eigenfunctions in cylindrical domains. However, for general domains $ \Omega $, explicit eigenfunctions may not be available in closed form. In these cases, alternative bases such as orthonormal polynomials or wavelets, which are defined algebraically rather than through an eigenvalue problem, may be used. These bases form a complete system for sufficiently regular $ \Omega' $ that contains $ \Omega $. A Sobolev-type space can then be defined using these algebraic bases and retain the theoretical guarantees established in this work.

\subsection{Comparison to Function Generation in Applied  Literature}\label{sec:smoothness-appl-dicsussion}

Next, we discuss how our assumption that the initial wave $\psi_0$ lies in $\Hcal^s(\Omega)$ is implicit in the function generation strategies commonly used in applied operator learning. In the applied literature, input functions are typically sampled from a Gaussian measure,  $ \text{N}(0, (-\Delta+  I)^{-\beta})$, or through some elementary push-forward of this distribution. This distribution, widely used in the applied stochastic PDE literature \citep{lord2014introduction}, was first introduced in the operator learning setting by \cite{bhattacharya2021model} and has since been implemented in works such as \citep{li2020fourier, kovachki2023neural}.  

Let $ (\lambda_j, \phi_j)_{j=1}^{\infty} $ be the eigenpairs of $ -\Delta $ in $ \Omega $ with the given boundary conditions. By the Spectral Mapping Theorem, the eigenvalues of the covariance operator $ (-\Delta+ I)^{-\beta} $ are $ (\lambda_j+ 1 )^{-\beta} $, while the eigenfunctions remain $ \phi_j $'s. Applying the Karhunen-Loève Theorem \citep[Theorem 7.3.5]{hsing2015theoretical}, a sample $u \sim \text{Normal}(0, (-\Delta+ I)^{-\beta})$ drawn from this distribution has the decomposition  
\[
u(x) = \sum_{j=1}^{\infty} (\lambda_j+ 1 )^{-\beta/2} \, \xi_j \, \phi_j(x),
\]
where $ \{\xi_j\}_{j=1}^{\infty} $ are uncorrelated standard Gaussian random variables, meaning $ \xi_j \sim \text{Normal}(0,1) $ and $ \mathbb{E}[\xi_i \xi_j] = \mathbf{1}[i = j] $. Thus, sampling $ u $ is reduced to generating a sequence of independent Gaussian random variables $ (\xi_j)_{j=1}^{\infty} $. In practical implementations, this is done by truncating the sequence to $ (\xi_j)_{j=1}^{M} $.

Using this decomposition, we compute  
\[
\mathbb{E}[|\langle u, \phi_j \rangle_{L^2}|^2] = (\lambda_j+ 1 )^{-\beta}\, \mathbb{E}[|\xi_j|^2] = (\lambda_j+ 1 )^{-\beta}.
\]
Thus, for any $s>0$,
\[
\mathbb{E} \left[ \|u\|_{\mathcal{H}^s}^2 \right]
= \mathbb{E} \left[ \sum_{j=1}^{\infty} (1+\lambda_j)^s |\langle u, \phi_j \rangle_{L^2}|^2 \right]
= \sum_{j=1}^{\infty} (1+\lambda_j)^{s-\beta}.
\]
Since $\lambda_j \to \infty$ as $j \to \infty$, this series converges whenever $s < \beta$. Consequently,
\[
\mathbb{E}[\|u\|_{\mathcal{H}^s}^2] < \infty \quad \text{for all } s < \beta.
\]
That is, the samples belong to $\mathcal{H}^s(\Omega)$ in expectation. For a more detailed discussion of how $\beta$ and $s$ relate when $\Omega = \mathbb{T}^d$, we refer the reader to \citep[Section B.1]{subedi2024benefits}.  

A similar strategy is used in \cite{lu2021learning}, where functions are sampled from a Gaussian process with a covariance kernel given by the radial basis function (RBF) kernel,  $
k(x,y)= \exp\left(-\frac{\|x-y\|_2^2}{2\sigma^2}\right).$ Since the eigenvalues of the RBF kernel decay exponentially fast, a similar analysis shows that the sampled functions belong to an extremely smooth space, essentially corresponding to the limiting case $ s = \infty $. This argument is not specific to the RBF kernel--any kernel with sufficiently fast eigenvalue decay produces functions with high regularity.  

Therefore, the requirement that input wave functions $ \psi $ belong to a smooth space $ \mathcal{H}^s(\Omega) $ is both reasonable and consistent with what is often an implicit assumption in applied operator learning literature.

\section{Proof of Proposition \ref{prop:unitarity}}\label{appdx:unitarity}
\begin{proof}
    Let $ u, v \in L^2(\Omega) $. Expanding the inner product,

    \begin{equation*}
        \begin{split}
            \langle \widehat{\operatorname{F}}_n(u), \widehat{\operatorname{F}}_n(v) \rangle_{L^2} &= \left\langle \sum_{|k|_{\infty}\leq K_n} w_k \langle u, \varphi_k \rangle_{L^2}, \sum_{|k|_{\infty}\leq K_n} w_k \langle v, \varphi_k \rangle_{L^2} \right\rangle_{L^2} \\
            &= \sum_{|k|_{\infty}, |\ell|_{\infty} \leq K_n} \langle u, \varphi_k \rangle_{L^2} \, \, \overline{\langle v, \varphi_{\ell} \rangle}_{L^2}\,\,  \langle w_k, w_{\ell} \rangle_{L^2}.
        \end{split}
    \end{equation*}

    Using the assumption on the PDE solver, we have  
    \[
    \langle w_k, w_{\ell} \rangle_{L^2} = \langle \operatorname{P}(\varphi_k), \operatorname{P}(\varphi_{\ell}) \rangle_{L^2} = \langle \varphi_k, \varphi_{\ell} \rangle_{L^2} = \mathbf{1}[k = \ell].
    \]

Substituting this into the sum,

    \begin{equation*}
        \begin{split}
            \langle \widehat{\operatorname{F}}_n(u), \widehat{\operatorname{F}}_n(v) \rangle_{L^2} &= \sum_{|k|_{\infty}, |\ell|_{\infty} \leq K_n} \langle u, \varphi_k \rangle_{L^2}\,\, \overline{\langle v, \varphi_{\ell} \rangle}_{L^2}\,\, \mathbf{1}[k = \ell] \\
            &= \sum_{|k|_{\infty} \leq K_n} \langle u, \varphi_k \rangle_{L^2} \,\, \overline{\langle v, \varphi_k \rangle}_{L^2}.
        \end{split}
    \end{equation*}

Using Parseval’s identity, we can rewrite this as  

    \begin{equation*}
        \begin{split}
            \sum_{|k|_{\infty} \leq K_n} \langle u, \varphi_k \rangle_{L^2}\,  \overline{\langle v, \varphi_k \rangle}_{L^2} &= \sum_{k \in \mathbb{Z}^d} \langle u, \varphi_k \rangle_{L^2}\,\, \overline{\langle v, \varphi_k \rangle}_{L^2} - \sum_{|k|_{\infty} > K_n} \langle u, \varphi_k \rangle_{L^2} \,\, \overline{\langle v, \varphi_k \rangle}_{L^2} \\
            &= \langle u, v \rangle_{L^2} - \sum_{|k|_{\infty} > K_n} \langle u, \varphi_k \rangle_{L^2} \, \overline{\langle v, \varphi_k \rangle}_{L^2}.
        \end{split}
    \end{equation*}

Property (i) follows since the second summation vanishes when $ u, v $ belong to the span of $ \{\varphi_k \, :\,  k \in \mathbb{Z}^d, |k|_{\infty} \leq K_n\} $.  To establish property (ii), setting $ u = v $ in the above expression,

\[
    \|\widehat{\operatorname{F}}_n(u)\|_{L^2}^2 = \|u\|_{L^2}^2 - \sum_{|k|_{\infty} > K_n} |\langle u, \varphi_k \rangle_{L^2}|^2 \leq \|u\|_{L^2}^2.
\]
This completes the proof.
\end{proof}

\section{Proof of Theorem \ref{thm:error}}\label{appdx:thm}

\begin{proof}
Recall that 
\[\widehat{\operatorname{F}}_n = \sum_{|k|_{\infty} \leq K_n} w_k \otimes \varphi_k.\]
For each $k$ such that $|k|_{\infty} \leq K_n$, define the error term
\[\delta_k := w_k- \operatorname{F}(\varphi_k).\]
By Assumption \ref{assumption}, it follows that $\norm{\delta_k}_{L^2 } \leq \varepsilon$. So, for any wave function $\psi \in \Hcal^{s}(\torus^d)$, we can expand
\begin{equation*}
    \begin{split}
       \widehat{\operatorname{F}}_n(\psi) &= \sum_{|k|_{\infty} \leq K_n} w_k \inner{\psi}{\varphi_k}_{L^2}  \\
       &=  \sum_{|k|_{\infty} \leq K_n} \operatorname{F}(\varphi_k)\inner{\psi}{\varphi_k}_{L^2}+ \sum_{|k|_{\infty} \leq K_n} \delta_k\, \inner{\psi}{\varphi_k}_{L^2}\\
       &= \operatorname{F} \left(\sum_{|k|_{\infty} \leq K_n} \varphi_k \inner{\psi}{\varphi_k}_{L^2} \right) + \sum_{|k|_{\infty} \leq K_n} \delta_k\, \inner{\psi}{\varphi_k}_{L^2},
    \end{split}
\end{equation*}
where the last equality follows from the linearity of $\operatorname{F}$. Then, applying the triangle inequality,
\begin{equation*}
    \begin{split}
        \norm{\widehat{\operatorname{F}}_n(\psi) - \operatorname{F}(\psi)}_{L^2} &= \norm{\operatorname{F} \left(\sum_{|k|_{\infty} \leq K_n} \varphi_k \inner{\psi}{\varphi_k}_{L^2} \right) + \sum_{|k|_{\infty} \leq K_n} \delta_k\, \inner{\psi}{\varphi_k}_{L^2} - \operatorname{F}(\psi)}_{L^2}\\
        &= \norm{\operatorname{F}\left(\sum_{|k|_{\infty} \leq K_n} \varphi_k \inner{\psi}{\varphi_k}_{L^2} - \psi \right) + \sum_{|k|_{\infty} \leq K_n} \delta_k\, \inner{\psi}{\varphi_k}_{L^2}}_{L^2}\\
        &\leq \norm{\operatorname{F}\left(\sum_{|k|_{\infty} \leq K_n} \varphi_k \inner{\psi}{\varphi_k}_{L^2} - \psi \right) }_{L^2} + \norm{ \sum_{|k|_{\infty} \leq K_n} \delta_k\, \inner{\psi}{\varphi_k}_{L^2}}_{L^2},
    \end{split}
\end{equation*}
 To bound the first term, note that the operator norm of $\operatorname{F}$ from a $L^2$ to $L^2$ is $1$ as $\operatorname{F}$ is a unitary operator. So, 
\begin{equation*}
    \begin{split}
        \norm{\operatorname{F}\left(\sum_{|k|_{\infty} \leq K_n} \varphi_k \inner{\psi}{\varphi_k}_{L^2} - \psi \right) }_{L^2} 
        &\leq    \norm{\sum_{|k|_{\infty} \leq K_n} \varphi_k \inner{\psi}{\varphi_k}_{L^2} - \psi }_{L^2}\\
        &= \sqrt{\sum_{|k|_{\infty} > K_n} |\inner{\psi}{\varphi_k}_{L^2}|^2 } \\
        &= \sqrt{\sum_{|k|_{\infty} > K_n} \frac{(1+ |k|_{2}^{2})^s }{(1+ |k|_{2}^{2})^s }\,\, |\inner{\psi}{\varphi_k}_{L^2}|^2 }\\
        &\leq \sqrt{\frac{1}{(1+ K_n^{2})^s}}\, \sqrt{\sum_{|k|_{\infty} > K_n} (1+ |k|_{2}^{2})^s \,\, |\inner{\psi}{\varphi_k}_{L^2}|^2 }\\
        &\leq K_n^{-s} \,\, \norm{\psi}_{\Hcal^s}.
    \end{split}
\end{equation*}
The first equality follows from Parseval's identity.

To bound the contribution from the PDE solver error, we use triangle inequality to get
\begin{equation*}
    \begin{split}
        \norm{ \sum_{|k|_{\infty} \leq K_n} \delta_k\, \inner{\psi}{\varphi_k}_{L^2}}_{L^2}  &\leq \sum_{|k|_{\infty} \leq K_n} \norm{\delta_k}_{L^2}\, |\inner{\psi}{\varphi_k}_{L^2}|\\
        &\leq \varepsilon \, \sum_{|k|_{\infty} \leq K_n} \, |\inner{\psi}{\varphi_k}_{L^2}|\\
        &\leq \varepsilon \, \sum_{|k|_{\infty} \leq K_n} \,
        \sqrt{\frac{(1+ |k|_{2}^{2})^s }{(1+ |k|_{2}^{2})^s }}\,  |\inner{\psi}{\varphi_k}_{L^2}|\\
        &\leq \varepsilon \sqrt{  \sum_{|k|_{\infty} \leq K_n} \,\frac{1}{(1+ |k|_{2}^{2})^s }}\, \sqrt{ \sum_{|k|_{\infty} \leq K_n} \,(1+ |k|_{2}^{2})^s \,  |\inner{\psi}{\varphi_k}_{L^2}|^2}\\
        &\leq \varepsilon \norm{\psi}_{\Hcal^s}\, \sqrt{  \sum_{|k|_{\infty} \leq K_n} \,\frac{1}{(1+ |k|_{2}^{2})^s }}\\
        &= \varepsilon \norm{\psi}_{\Hcal^s} \gamma_n,
        \end{split}
\end{equation*}
where 
\[\gamma_n :=  \sqrt{  \sum_{|k|_{\infty} \leq K_n} \,\frac{1}{(1+ |k|_{2}^{2})^s }}.\]
Thus, we have established that
\[\norm{\widehat{\operatorname{F}}_n(\psi) - \operatorname{F}(\psi)}_{L^2}  \leq \norm{\psi}_{\Hcal^s} \left(\varepsilon \, \gamma_n + K_n^{-s}  \right). \]
Note that $K_n = (n^{1/d}-1)/2 \geq n^{1/d}/3$ as long as  $n \geq 3^d$. This yields that $K_n^{-s} \leq 3^s \, n^{-s/d}$.

To bound $\gamma_n$, recall that $|\{k \in \integers^d \, :\, |k|_{\infty}=j\}| = 2(2j+1)^{d-1}$. This is because one of the entry of $m$ has to be $\pm j$ and other $d-1$ entries could be anything in $\{-j \ldots, -1, 0, 1, \ldots, j\}$.  Thus, 
\begin{equation*}
    \begin{split}
        \gamma_n^2 = \sum_{|k|_{\infty} \leq K_n} \frac{1}{(1+|k|_{2}^{2})^s} \leq \sum_{|k|_{\infty} \leq K_n} \frac{1}{(1+|k|_{\infty}^{2})^s} &\leq  1 +  \sum_{1<|k|_{\infty} \leq K_n} \frac{1}{(1+|k|_{\infty}^{2})^s} \\
        &\leq 1+ \sum_{j=1}^{K_n}\frac{2 (2j+1)^{d-1}}{(1+j^{2})^s}\\
        &\leq 1 + \sum_{j=1}^{K_n}\frac{2 (2j+1)^{d-1}}{j^{2s}} \\
        &\leq 1+ 2\cdot 3^{d-1} \sum_{j=1}^{K_n} \frac{1}{j^{2s-d+1}}\\
        &\lesssim  \int_{1}^{K_n} \frac{1}{t^{2s-d+1}} \, dt\\
        &\lesssim \begin{cases}
        1, & \text{\emph{if} } 2s > d,\\[6pt]
     \log(K_n),       & \text{\emph{if} } 2s = d,\\[6pt]
    K_n^{d-2s}  & \text{\emph{if} } 2s < d.
  \end{cases}
    \end{split}
\end{equation*} 
Our proof completes upon noting that $K_n \lesssim n^{1/d}$.
\end{proof}

\section{Proof of Theorem \ref{thm:error-lower}}\label{appdx:error-lower}
\begin{proof}
We will break down the proof into multiple steps.  \\

\noindent \textbf{Defining the Hamiltonian:} 
Considering the case where the potential is zero, meaning $ V(x) = 0 $ for all $ x \in \mathbb{T}^d $. In this case, the Hamiltonian simplifies to  
\[
\operatorname{H} = -\frac{\hbar^2}{2m} \Delta.
\]
The corresponding solution operator is given by  
\[
\operatorname{F} = \exp\left(-\frac{i}{\hbar} T \operatorname{H} \right) = \exp\left(i \frac{\hbar T}{2m} \Delta \right).
\]
Note that the Fourier modes are eigenfunctions of this operator. Specifically, for any wave vector $ k $, we have  
\[
\operatorname{F} (\varphi_k) =  \exp\left(- i \frac{4\pi^2 |k|_2^2 \hbar T}{2m} \right) \varphi_k.
\]
This follows from expanding the operator exponential and using the property  
\[
\Delta \varphi_k = \sum_{j=1}^d \frac{\partial^2 e^{2\pi i k \cdot x} }{\partial x_j^2} = \sum_{j=1}^d (2\pi i)^2 k_j^2 e^{2\pi i k \cdot x} = -4\pi^2 |k|_2^2 \varphi_k.
\]
Defining  
\[
\eta_k := \exp\left(- i \frac{4\pi^2 |k|_2^2 \hbar T}{2m} \right),
\]
we can express the action of the solution operator as  

\[
\operatorname{F}(\varphi_k) = \eta_k \varphi_k.
\]\\

\noindent \textbf{Specifying a PDE Solver.}    Our next step is to specify the exact details of the PDE solver that satisfies the $ \varepsilon $-approximate assumption while also allowing us to construct hard instances to establish the lower bound. To that end, let $ \operatorname{P} $ be the PDE solver defined as  

\[
\operatorname{P}(u) = \operatorname{F}(u) + \varepsilon \varphi_0, \quad \text{for every } u \in L^2(\Omega).
\]
Here, $ \varphi_0 $ is simply the constant function $ 1 $ on the domain. That is, our PDE solver oracle returns the true solution shifted by $ \varepsilon \varphi_0 $ noise. While such a PDE solver is not practical, it is still a valid $ \varepsilon $-approximate oracle since $ \|\varphi_0\|_{L^2} = 1 $, and thus our upper bound in Theorem \ref{thm:error} applies.  \\

\noindent \textbf{Writing out the Estimator.}
For this solution operator and the PDE solver specified above, our estimator has a more concrete form. In particular, we can write
\[\widehat{\operatorname{F}} := \sum_{|k|_{\infty} \leq K_n} \eta_k\,  \varphi_k \otimes \varphi_k + \varepsilon \sum_{|k|_{\infty} \leq K_n}  \varphi_0 \otimes \varphi_k.\]\\

\noindent \textbf{Defining the Test Function.}
Given a sample size budget of $ n $, we now construct a hard test wave function $ \psi_{\text{test}} $ to establish the claimed lower bound. To do this, choose a large $ M \gg n $, which will be specified later, and define $ \psi_{\text{test}} $ as  

\[
\psi_{\text{test}} = \sum_{ |k|_{\infty} \leq M} c_k \varphi_k,
\]
for some coefficients $ c_k \geq 0 $.

\noindent \textbf{Establishing the Lower Bound.}
For the wave function $ \psi_{\text{test}} $ defined above, the true evolution under $ \operatorname{F} $ is given by  

\[
\operatorname{F}(\psi_{\text{test}}) = \sum_{|k|_{\infty} \leq M} c_k \, \eta_k \, \varphi_k.
\]
On the other hand, the estimator $ \widehat{\operatorname{F}}_n $ produces  
\[
\widehat{\operatorname{F}}_n(\psi_{\text{test}}) = \sum_{|k|_{\infty} \leq K_n} c_k \, \eta_k \, \varphi_k + \varepsilon \sum_{|k|_{\infty} \leq K_n} c_k \, \varphi_0.
\]
Rewriting the second term,

\[
\widehat{\operatorname{F}}_n(\psi_{\text{test}}) = \sum_{|k|_{\infty} \leq K_n} c_k \, \eta_k \, \varphi_k + \varepsilon \left(\sum_{|k|_{\infty} \leq K_n} c_k \right) \varphi_0.
\]
Thus, the difference between the estimated and true evolution is

\[
\widehat{\operatorname{F}}_n(\psi_{\text{test}}) - \operatorname{F}(\psi_{\text{test}}) = \varepsilon \left(\sum_{|k|_{\infty} \leq K_n} c_k \right) \varphi_0 - \sum_{K_n < |k|_{\infty} \leq M} c_k \, \eta_k \, \varphi_k.
\]

Using Parseval’s identity, we obtain
\[
\begin{split}
    \norm{\widehat{\operatorname{F}}_n(\psi_{\text{test}}) - \operatorname{F}(\psi_{\text{test}})}^2_{L^2}
    &= \varepsilon^2 \left|\sum_{|k|_{\infty} \leq K_n} c_k \right|^2 + \sum_{K_n < |k|_{\infty} \leq  M} |c_k \eta_k|^2 \\
    &= \varepsilon^2 \left(\sum_{|k|_{\infty} \leq K_n} |c_k|\right)^2 + \sum_{K_n < |k|_{\infty} \leq M} |c_k|^2,
\end{split}
\]
where the second equality follows from the assumptions that $ c_k \geq 0 $ and $ |\eta_k| = 1 $ for all $ k \in \mathbb{Z}^d $.  

 To establish the claimed rate, we now choose $ c_k $ appropriately while ensuring that $ \|\psi_{\text{test}}\|_{L^2} = 1 $ and $ \|\psi_{\text{test}}\|_{\mathcal{H}^s} $ is bounded by an absolute constant. Fix an index $\ell \in \mathbb{Z}^d$ such that
\[
|\ell|_{\infty} = \lceil K_n+1\rceil,
\]
and set
\[
c_\ell := \frac{1}{\sqrt{2}(1+|\ell|_2^2)^{s/2}}.
\]
Also define
\[
A_n := \sum_{0<|k|_{\infty}\leq K_n} (1+|k|_2^2)^{-s}.
\]
For each $k$ with $0<|k|_{\infty}\leq K_n$, set
\[
c_k := \frac{(1+|k|_2^2)^{-s}}{\sqrt{2\,A_n}},
\]
and set $c_k=0$ for all \emph{other} $|k|_{\infty}>K_n$. Finally, choose $c_0\geq 0$ so that
\[
\|\psi_{\text{test}}\|_{L^2}=1.
\]

We first verify that this choice is valid. Since
\[
\sum_{0<|k|_{\infty}\leq K_n} c_k^2
=
\frac{1}{2A_n}\sum_{0<|k|_{\infty}\leq K_n}(1+|k|_2^2)^{-2s}
\leq
\frac{1}{2A_n}\sum_{0<|k|_{\infty}\leq K_n}(1+|k|_2^2)^{-s}
=\frac{1}{2},
\]
and $c_\ell^2 < 1/2$ for sufficiently large $n$, we can choose $c_0\in[0,1]$ so that $\|\psi_{\text{test}}\|_{L^2}=1$.

Next, we bound the Sobolev norm. By construction,
\[
\begin{split}
\|\psi_{\text{test}}\|_{\mathcal{H}^s}^2
&=
|c_0|^2
+
\sum_{0<|k|_{\infty}\leq K_n}(1+|k|_2^2)^s c_k^2
+
(1+|\ell|_2^2)^s c_\ell^2\\
&=
|c_0|^2
+
\frac{1}{2A_n}\sum_{0<|k|_{\infty}\leq K_n}(1+|k|_2^2)^{-s}
+\frac{1}{2}\\
&=
|c_0|^2+\frac{1}{2} + \frac{1}{2}\\
&\leq 2.
\end{split}
\]
Thus, $\|\psi_{\text{test}}\|_{\mathcal{H}^s}$ is bounded by an absolute constant.

We now lower bound the two terms in
\[
\norm{\widehat{\operatorname{F}}_n(\psi_{\text{test}})-\operatorname{F}(\psi_{\text{test}})}_{L^2}^2
=
\varepsilon^2\left(\sum_{|k|_{\infty}\leq K_n}|c_k|\right)^2
+
\sum_{K_n<|k|_{\infty}\leq M}|c_k|^2.
\]

First, for the tail term, since $|\ell|_{\infty}=\lceil K_n+1\rceil$, we have
\[
\sum_{K_n<|k|_{\infty}\leq M}|c_k|^2
\geq
|c_\ell|^2
=
\frac{1}{2(1+|\ell|_2^2)^s}
\gtrsim
K_n^{-2s}
\gtrsim
n^{-2s/d},
\]
where we use the fact that $K_n \lesssim  n^{1/d}$.

Second, for the solver-error term,
\[
\sum_{|k|_{\infty}\leq K_n}|c_k|
\geq
\sum_{0<|k|_{\infty}\leq K_n} |c_k|
=
\frac{1}{\sqrt{2A_n}}
\sum_{0<|k|_{\infty}\leq K_n}(1+|k|_2^2)^{-s}
=
\sqrt{A_n/2}.
\]
Therefore,
\[
\varepsilon^2\left(\sum_{|k|_{\infty}\leq K_n}|c_k|\right)^2
\geq
\varepsilon^2 \frac{A_n}{2}.
\]

It remains to estimate $A_n$. For each $m \geq 1$, define
\[
S_m := \{k \in \mathbb{Z}^d : |k|_\infty = m\}.
\]

Note that $|S_m| \geq (2m+1)^{d-1} \gtrsim m^{d-1}$. Indeed, to construct an element of $S_m$, we may fix one coordinate to be either $m$ or $-m$, while allowing the remaining $d-1$ coordinates to vary arbitrarily over $\{-m,\ldots,m\}$. This yields at least a constant multiple of $m^{d-1}$ distinct lattice points. Moreover, for every $k \in S_m$, we have
\[
m \le |k|_2 \le \sqrt{d}\, m.
\]
Therefore,
\[
(1+|k|_2^2)^{-s} \gtrsim (1+m^2)^{-s} \gtrsim m^{-2s}
\qquad \text{for all } k \in S_m,
\]
and hence
\[
A_n
= \sum_{m=1}^{K_n} \sum_{k \in S_m} (1+|k|_2^2)^{-s}
\gtrsim
\sum_{m=1}^{K_n} |S_m| (1+m^2)^{-s}
\gtrsim
\sum_{m=1}^{K_n} m^{d-1}(1+m^2)^{-s}\gtrsim \sum_{m=1}^{K_n} m^{d-1-2s}.
\]

A standard integral comparison then yields
\[
A_n\gtrsim
\begin{cases}
1, & 2s>d,\\[4pt]
\log K_n, & 2s=d,\\[4pt]
K_n^{d-2s}, & 2s<d.
\end{cases}
\]

Hence,
\[
\norm{\widehat{\operatorname{F}}_n(\psi_{\text{test}})-\operatorname{F}(\psi_{\text{test}})}_{L^2}^2
\gtrsim
\begin{cases}
\varepsilon^2 + K_n^{-2s}, & \text{if } 2s>d,\\[4pt]
\varepsilon^2 \log K_n + K_n^{-2s}, & \text{if } 2s=d,\\[4pt]
\varepsilon^2 K_n^{d-2s} + K_n^{-2s}, & \text{if } 2s<d.
\end{cases}
\]
Since $K_n \asymp n^{1/d}$, this becomes
\[
\norm{\widehat{\operatorname{F}}_n(\psi_{\text{test}})-\operatorname{F}(\psi_{\text{test}})}_{L^2}^2
\gtrsim
\begin{cases}
\varepsilon^2 + n^{-2s/d}, & \text{if } 2s>d,\\[4pt]
\varepsilon^2 \log n + n^{-2s/d}, & \text{if } 2s=d,\\[4pt]
\varepsilon^2 n^{1-\frac{2s}{d}} + n^{-2s/d}, & \text{if } 2s<d.
\end{cases}
\]

Finally, using
\[
\sqrt{a^2+b^2}\geq \frac{1}{\sqrt{2}}(a+b), \qquad a,b\geq 0,
\]
we obtain
\[
\norm{\widehat{\operatorname{F}}_n(\psi_{\text{test}})-\operatorname{F}(\psi_{\text{test}})}_{L^2}
\gtrsim
\begin{cases}
\varepsilon + n^{-s/d}, & \text{if } 2s>d,\\[6pt]
\varepsilon \sqrt{\log n} + n^{-s/d}, & \text{if } 2s=d,\\[6pt]
\varepsilon \, n^{\frac12-\frac{s}{d}} + n^{-s/d}, & \text{if } 2s<d.
\end{cases}
\]
This completes the proof.
\end{proof}

\section{Refined Upper Bound Under Stronger Assumptions on PDE Solver.}\label{appdx:thm-improved}

\begin{proof}[Proof of Theorem \ref{thm:improved}]
  Our proof here largely follow the proof of Theorem \ref{thm:error} provided in Appendix \ref{appdx:thm}. Recall that, for any wave function $\psi$, we established in the proof of Theorem \ref{thm:error} that
  \[  \norm{\widehat{\operatorname{F}}_n(\psi) - \operatorname{F}(\psi)}_{L^2} \leq \norm{\operatorname{F}\left(\sum_{|k|_{\infty} \leq K_n} \varphi_k \inner{\psi}{\varphi_k}_{L^2} - \psi \right) }_{L^2} + \norm{ \sum_{|k|_{\infty} \leq K_n} \delta_k\, \inner{\psi}{\varphi_k}_{L^2}}_{L^2}. 
\]
The first term does not have any randomness. So, following the same argument as in that proof, we can show that
\[\norm{\operatorname{F}\left(\sum_{|k|_{\infty} \leq K_n} \varphi_k \inner{\psi}{\varphi_k}_{L^2} - \psi \right) }_{L^2} \leq  K_n^{-s} \,\, \norm{\psi}_{\Hcal^s}\leq 3^{s} \, c\, n^{-\frac{s}{d}}.\]
Here, we used the definition of $K_n$ and the fact that $\norm{\psi}_{\Hcal^s}\leq c$. Now, it remains to bound the term with $\delta_k$'s. Since this is a random variable, we want to bound its expectation. To that end, Jensen's inequality implies
\[\expect\left[\norm{ \sum_{|k|_{\infty} \leq K_n} \delta_k\, \inner{\psi}{\varphi_k}_{L^2}}_{L^2} \right] \leq \sqrt{\expect\left[\norm{ \sum_{|k|_{\infty} \leq K_n} \delta_k\, \inner{\psi}{\varphi_k}_{L^2}}_{L^2}^2 \right]}.\]
Note that  
\begin{equation*}
    \begin{split}
        \norm{ \sum_{|k|_{\infty} \leq K_n} \delta_k\, \inner{\psi}{\varphi_k}_{L^2}}_{L^2}^2 &= \inner{\sum_{|k|_{\infty} \leq K_n} \delta_k\, \inner{\psi}{\varphi_k}_{L^2}}{\sum_{|k|_{\infty} \leq K_n} \delta_k\, \inner{\psi}{\varphi_k}_{L^2}}_{L^2}\\
        &= \sum_{|k|_{\infty},|\ell|_{\infty} \leq K_n} \inner{\psi}{\varphi_k}_{L^2} \, \overline{\inner{\psi}{\varphi_{\ell}}_{L^2}} \inner{\delta_k}{\delta_{\ell}}_{L^2}\\
        &= \sum_{|k|_{\infty}\leq K_n} | \inner{\psi}{\varphi_k}_{L^2}|^2 \norm{\delta_k}_{L^2}^2 + \sum_{k \neq \ell} \inner{\psi}{\varphi_k}_{L^2} \, \overline{\inner{\psi}{\varphi_{\ell}}_{L^2}} \inner{\delta_k}{\delta_{\ell}}
    \end{split}
\end{equation*}
Note that the cross terms $k \neq \ell$ vanishes in expectation due to part (ii) of Assumption \ref{assumption:random}. Using part (i) of Assumption \ref{assumption:random} yields
\begin{equation*}
    \begin{split}
        \expect\left[ \norm{ \sum_{|k|_{\infty} \leq K_n} \delta_k\, \inner{\psi}{\varphi_k}_{L^2}}_{L^2}^2 \right] &=  \sum_{|k|_{\infty}\leq K_n} | \inner{\psi}{\varphi_k}_{L^2}|^2 \expect[\norm{\delta_k}_{L^2}^2]\\
        &\leq \varepsilon^2 \sum_{|k|_{\infty}\leq K_n}|\inner{\psi}{\varphi_k}_{L^2}|^2  \\
        &\leq \varepsilon^2,
    \end{split}
\end{equation*}
where the final step uses the fact that $\norm{\psi}_{L^2}^2 =1$. This shows that 
\[\expect\left[\norm{ \sum_{|k|_{\infty} \leq K_n} \delta_k\, \inner{\psi}{\varphi_k}_{L^2}}_{L^2} \right] \leq \varepsilon.\]
This completes our proof. 
\end{proof}

\section{Proof of Theorem \ref{thm:time-gen}}\label{appdx:time-gen}
\begin{proof}

Note that  

\[
\widehat{\operatorname{F}}_n^q \psi - \operatorname{F}^q \psi = (\widehat{\operatorname{F}}_n^q - \operatorname{F}^q) \psi = \sum_{j=0}^{q-1} \widehat{\operatorname{F}}_n^{q-1-j} (\widehat{\operatorname{F}}_n - \operatorname{F}) \operatorname{F}^j \psi.
\]

Applying the triangle inequality,

\[
\|\widehat{\operatorname{F}}_n^q \psi - \operatorname{F}^q \psi\|_{L^2} \leq \sum_{j=0}^{q-1} \left\| \widehat{\operatorname{F}}_n^{q-1-j} (\widehat{\operatorname{F}}_n - \operatorname{F}) \operatorname{F}^j \psi \right\|_{L^2}.
\]

Using property (ii) of Proposition \ref{prop:unitarity} iteratively $ q-j-1 $ times, we obtain

\[
\left\| \widehat{\operatorname{F}}_n^{q-1-j} (\widehat{\operatorname{F}}_n - \operatorname{F}) \operatorname{F}^j \psi \right\|_{L^2} \leq \left\| (\widehat{\operatorname{F}}_n - \operatorname{F}) \operatorname{F}^j \psi \right\|_{L^2}.
\]

Furthermore, applying Theorem \ref{thm:error}, we obtain the bound

\[
\left\| (\widehat{\operatorname{F}}_n - \operatorname{F}) \operatorname{F}^j \psi \right\|_{L^2} \leq \|\operatorname{F}^j \psi\|_{\Hcal^s} \left( \varepsilon \gamma_n + 3^s n^{-s/d} \right).
\]

Thus, we conclude that

\[
\|\widehat{\operatorname{F}}_n^q \psi - \operatorname{F}^q \psi\|_{L^2} 
\leq \left( \varepsilon \gamma_n + 3^s n^{-s/d} \right) \sum_{j=0}^{q-1} \|\operatorname{F}^j \psi\|_{\Hcal^s}.
\]
\end{proof}

\section{Proof of Corollary \ref{cor:smooth-potential}}\label{appdx:smooth-potential}

\subsection{Proof of Part (i)}
\begin{proof}
    Let $ V(x) = a $ for all $ x \in \mathbb{T}^d $. Then, for every Fourier mode $ \varphi_k $, the Hamiltonian acts as  
\[
\operatorname{H} \varphi_k = \left(-\frac{\hbar^2}{2m} \Delta + V(\cdot)\right) \varphi_k = \left(\frac{\hbar^2}{2m} \, 4\pi^2|k|_2^2 + a \right) \varphi_k.
\]
The second equality holds because $\varphi_k$ is an eigenfunction of $-\Delta$ with eigenvalue $4\pi^2 |k|_2^2$. Next, applying the time evolution operator, we get  
\[
\operatorname{F}(\varphi_k) = e^{-\frac{\imaginary}{\hbar} T \operatorname{H}} \varphi_k = e^{-\frac{\imaginary}{\hbar} T \left(\frac{\hbar^2}{2m} \, 4\pi^2|k|_2^2 + a \right)} \varphi_k.
\]
Since the modulus of the complex exponential factor is always one, we can use this identity to establish that 
\begin{equation*}
    \begin{split}
      \|\operatorname{F}(\psi)\|_{\mathcal{H}^s} &=  \sqrt{\sum_{k \in \mathbb{Z}^d} (1+ |k|_{2}^{2})^s \, |\langle \operatorname{F}(\psi), \varphi_k \rangle_{L^2}|^2}\\
      &= \sqrt{\sum_{k \in \mathbb{Z}^d} (1+ |k|_{2}^{2})^s \left |\left\langle \sum_{\ell \in \integers^d} \inner{\psi}{\varphi_{\ell}}\operatorname{F}(\varphi_{\ell}), \varphi_{k} \right\rangle_{L^2} \right|^2}\\
      &= \sqrt{\sum_{k \in \mathbb{Z}^d} (1+ |k|_{2}^{2})^s \left |\left\langle \sum_{\ell \in \integers^d} \inner{\psi}{\varphi_{\ell}} e^{-\frac{\imaginary}{\hbar} T \left(\frac{\hbar^2}{2m} \, 4\pi^2|\ell|_2^2 + a \right)} \varphi_{\ell}, \varphi_k \right\rangle_{L^2} \right|^2}\\
      &= \sqrt{\sum_{k \in \mathbb{Z}^d} (1+ |k|_{2}^{2})^s\,|\inner{\psi}{\varphi_k}|^2 },
    \end{split}
\end{equation*}
where the final equality uses the fact that $\inner{\varphi_{\ell}}{\varphi_k}=\indicator[k=\ell]$ and $\left|e^{-\frac{\imaginary}{\hbar} T \left(\frac{\hbar^2}{2m} \, 4\pi^2|k|_2^2 + a \right)} \right|=1$.
Applying this iteratively for $ j $ steps, we obtain  $
\|\operatorname{F}^j(\psi)\|_{\mathcal{H}^s} = \|\psi\|_{\mathcal{H}^s}$ for all $j \in \mathbb{N}.$
\end{proof}
\subsection{Proof Part (ii)}
\begin{proof}

Our result follows directly from the bound in \cite[Theorem 1]{delort2010growth}, originally established by \cite{bourgain1999growth}, which states that
\[
\norm{\operatorname{F}^j \psi}_{\Hcal^s} = \norm{\psi(\cdot, jT)}_{\Hcal^s} \leq c\, (1 + jT)\, \norm{\psi}_{\Hcal^s}.
\]
This can be further refined using \cite[Equation 1.3]{delort2010growth}, yielding the bound
\[
\norm{\operatorname{F}^j \psi}_{\Hcal^s} \leq c\, (1 + jT)^{\delta}\, \norm{\psi}_{\Hcal^s}
\]
for any fixed $ \delta > 0 $. Substituting this into our generalization bound gives
\[
\norm{\widehat{\operatorname{F}}_n^q(\psi) - \operatorname{F}^q(\psi)}_{L^2}
\leq \norm{\psi}_{\Hcal^s} \left( \varepsilon \gamma_n + 3^s n^{-s/d} \right) \cdot c\, q (1 + Tq)^{\delta}.
\]
\end{proof}

\subsection{Proof Part (iii)}
\begin{proof}
Since the Hamiltonian $ \operatorname{H} $ is time-independent, the evolution operator satisfies  
\[
\operatorname{F}^j\psi = e^{-i jT \operatorname{H} / \hbar} \psi.
\]
Defining $ \psi(t) $ as the wave function at time $ t $ with initial condition $ \psi(0) = \psi $, we write  
\[
\operatorname{F}^j \psi = \psi(jT).
\]
Thus, bounding the Sobolev norm of $ \operatorname{F}^j\psi $ reduces to bounding $ \|\psi(t)\|_{\mathcal{H}^s} $ in terms of $ \|\psi(0)\|_{\mathcal{H}^s} $ for all $ t > 0 $. To proceed, define the operator  
\[
\Lambda^s := \left(\operatorname{I} - (4\pi^2)^{-1}\, \Delta\right)^{s/2}.
\]
Note that
\begin{equation*}
    \begin{split}
       \norm{\Lambda^s \psi}_{L^2}^2=\norm{ \sum_{k \in \integers^d} \inner{\psi}{\varphi_k}_{L^2} \Lambda^s \varphi_k }_{L^2}^2  &=\norm{ \sum_{k \in \integers^d} \inner{\psi}{\varphi_k}_{L^2} (1+ |k|_2^2 )^{s/2}\varphi_k}_{L^2}^2\\
       &= \sum_{k \in \integers^d} (1+|k|_2^2)^s |\inner{\psi}{\varphi_k}_{L^2}|^2 \\
       &= \norm{\psi}_{\Hcal^s}^2. 
    \end{split}
\end{equation*}
Thus, we focus on bounding $ \|\Lambda^s \psi(t)\|_{L^2} $.  

\noindent \textbf{Energy Functional.}
Define the energy functional  
\[
E_s(t) := \|\Lambda^s \psi(t)\|_{L^2}^2.
\]
Using the product rule rule in a Hilbert space, we obtain
\[
\frac{d}{dt} E_s(t) = \langle \Lambda^s (\partial_t \psi), \Lambda^s \psi \rangle_{L^2} + \langle \Lambda^s  \psi, \Lambda^s (\partial_t \psi) \rangle_{L^2} = 2 \operatorname{Re} \left(\langle \Lambda^s (\partial_t \psi), \Lambda^s \psi \rangle_{L^2}\right).
\]
Since the Schr\"{o}dinger equation states   
\[
\partial_t \psi = \imaginary \frac{\hbar}{2m} \Delta \psi - \frac{\imaginary}{\hbar} V \psi,
\]
applying $ \Lambda^s $ to both sides yields  
\[
\Lambda^s(\partial_t \psi) = \imaginary \frac{\hbar}{2m} \Lambda^s(\Delta \psi) - \frac{\imaginary}{\hbar} \Lambda^s(V \psi).
\]
Thus, we obtain the energy functional equation
\[
\frac{d}{dt} E_s(t) = 2 \operatorname{Re} \left(\left\langle \imaginary \frac{\hbar}{2m} \Lambda^s (\Delta \psi) - \frac{\imaginary}{\hbar} \Lambda^s (V \psi), \Lambda^s \psi \right\rangle_{L^2}\right).
\]

Note that  

\[
   \operatorname{Re} \left\langle
     \imaginary \frac{\hbar}{2m} \Lambda^s(\Delta \psi),
     \Lambda^s\psi
   \right\rangle
   =  \operatorname{Re} \left(  \imaginary \frac{\hbar}{2m} 
   \left\langle
    \Lambda^s(\Delta \psi),
     \Lambda^s\psi
   \right\rangle \right) = 0.
\]
This follows because $ \langle
    \Lambda^s(\Delta \psi),
     \Lambda^s\psi
   \rangle $ is a real number. To see why, observe that we can rewrite  
\[
\langle
    \Lambda^s(\Delta \psi),
     \Lambda^s\psi
   \rangle = \langle
    (\Lambda^s \Delta \Lambda^{-s}) \Lambda^s \psi,
     \Lambda^s\psi
   \rangle.
\]
Since $ (\Lambda^s \Delta \Lambda^{-s}) $ is a self-adjoint operator on $ L^2 $,  the inner product must be real. So, the only contribution comes from  
\[
-\frac{\imaginary}{\hbar} \Lambda^s(V \psi).
\]
Thus, we obtain  
\[
\frac{d}{dt} E_s(t) = -\frac{2}{\hbar} \operatorname{Im} \langle \Lambda^s (V \psi), \Lambda^s \psi \rangle_{L^2}. 
\]
Applying the Cauchy–Schwarz inequality,  
\[
\left| \frac{d}{dt} E_s(t) \right| \leq \frac{2}{\hbar} \|\Lambda^s (V \psi)\|_{L^2} \|\Lambda^s \psi\|_{L^2} = \frac{2}{\hbar} \|\Lambda^s(V \psi)\|_{L^2} \sqrt{E_s(t)}.
\]
\noindent \textbf{ Bounding the Sobolev Norm of $ V \psi $.}
By assumption, $ V $ belongs to $ \Hcal^{r}(\mathbb{T}^d) $. We will now establish  
\[
\|V \psi\|_{\mathcal{H}^s} \leq a \|V\|_{\Hcal^r} \|\psi\|_{\mathcal{H}^s}
\]
for some universal $a>0$ that only depends on $s, d, r$.  This is a Hölder-type inequality for the Sobolev norm of a product of two functions, commonly known as a Sobolev multiplication inequality. This inequality is established in the proof of  \citep[Theorem 5.1]{behzadan2021multiplication} for the domain $\mathbb{R}^d$ (take $p_1=p_2=p=2$, $s_1=r$, and $s_2=s$). The proof works verbatim for $\torus^d$ as it only uses the Sobolev Embedding Theorems, which continue to hold on $\torus^d$.


Thus, rewriting in terms of $ \Lambda^s $ yields 
\[
\|\Lambda^s (V \psi)\|_{L^2} \leq a \|V\|_{\Hcal^r} \|\psi\|_{\mathcal{H}^s},
\]
which upon using the definition of  energy functional implies
\[
\|\Lambda^s (V \psi)\|_{L^2} \leq a \|V\|_{\Hcal^r} \sqrt{E_s(t)}.
\]

\noindent \textbf{Applying Gr\"{o}nwall's Inequality.}
Substituting this inequality into our bound for $ \frac{d}{dt} E_s(t) $,  we obtain
\[
\left| \frac{d}{dt} E_s(t) \right| \leq \frac{2a}{\hbar } \|V\|_{\Hcal^r} E_s(t).
\]
Applying Gr\"{o}nwall's inequality on $[0,t]$, we obtain
\[
E_s(t) \leq E_s(0) \exp \left( \frac{2a}{\hbar} \|V\|_{\Hcal^r} \cdot t \right).
\]
Finally, using the equivalence of norms,
\[
\|\psi(t)\|_{\mathcal{H}^s}^2 \leq \|\psi(0)\|_{\mathcal{H}^s}^2 \exp \left( \frac{2a}{\hbar} \|V\|_{\Hcal^r} \, \cdot t \right).
\]
Taking square roots on both sides and defining $ c :=\frac{2a}{\hbar}$, we conclude
\[
\|\psi(t)\|_{\mathcal{H}^s} \leq \|\psi(0)\|_{\mathcal{H}^s} \exp \left( c\,  \|V\|_{\Hcal^r} \cdot t \right).
\]

\end{proof}

\section{Experimental Potentials Details}\label{appdx:exp_potentials}
We here provided more detailed descriptions of the potentials studied in the main text.

\noindent\textbf{Free Particle} If a particle is not exposed to an external potential, $V(x) =0$ for all $x \in \Omega$.

\noindent\textbf{Harmonic Oscillator} Molecular vibrations are naturally modeled with a potential $V(x) = \frac{1}{2} m \, \omega^2\, |x|_2^2$, where $m$ is the particle mass and $\omega$ the angular frequency of the oscillation. 

\noindent\textbf{Double Slit} For a particle traveling across a barrier of potential $V_0$ at $x=x_0$ with two slits centered at $y_1$ and $y_2$ each with width $w$ that are sufficiently far apart such that $|y_1 - y_2| \gg w$, the system potential is given by $
V(x,y) = V_0$  when $x = x_0 \text{ and } |y - y_1| > \frac{w}{2} \text{ and } |y - y_2| > \frac{w}{2}$, whereas $V(x,y)=0$ otherwise.

\noindent\textbf{Random Potentials} To demonstrate robustness over arbitrary smooth potentials, a random potential $V(x)$ was drawn from a Gaussian Random Field identically to how such draws were made to define initial conditions, with $\alpha=1$, $\beta=1$, and $\gamma=4$. 

\noindent\textbf{Coloumb Potential} For a particle exposed to a radially symmetric electric field, such as in a Hydrogen atom, the potential is given by $V(x) = -\frac{k e^2}{r^2}$. We specifically focus on the case of a fixed radius of $r=1$, for which the system can modeled as a uniform field in spherical coordinates. As discussed, both the pseudospectral solver and estimator were computed using spherical harmonics for this setup. 

\noindent\textbf{Paul Trap for Qubit Design} A Paul trap is a device that confines charged particles, such as ions, using oscillating electric fields. Notably, therefore, such a potential is time-dependent. For a detailed mathematical treatment of the Paul trap, see \citep[Chapter 2]{Major2005}. A broader discussion on how Paul traps are used to localize charged ions for qubit encoding in their energy states can be found in the review article by \citep{bernardini2023quantum}. In 2D, the potential function is given by  $
V(x,y,t) = \frac{U_0 + V_0 \cos(\omega t)}{r_0^2} \, (x^2 + y^2).$

\noindent\textbf{Shaken Lattice} Optical lattices are a common design pattern for trapping neutral atoms with laser interferometry \cite{deutsch2000quantum}. The promise in certain applications, such as quantum computing, is subsequent manipulation of such trapped atoms \cite{zhang2006manipulation}. One mechanism of control is known as ``shaking,'' in which the phase of the potentials is manipulated to affect the momenta of the trapped particles \cite{zheng2014floquet,kiely2016shaken,weidner2017atom}. If the shaking is restricted to a single axis, the potential is then given by $V(x,y,t) = V_0 \cos[k (x - A \sin(\omega t))] + V_0 \cos(k y)$.

\noindent\textbf{Pulsed Gaussian} Recent works have begun investigating the stability of bound states under pulsed external potentials, such as that of deuterons as studied in \cite{rais2022bound}. In particular, stability was assessed in the presence of external Gaussian pulses, given by the potential $V(x,y,t) = V_0 \exp\left(-\frac{(x - x_0)^2}{2\sigma_x^2} - \frac{(y - y_0)^2}{2\sigma_y^2} \right) \sum_{t_0} e^{-\frac{(t - t_0)^2}{2 \sigma_t^2}}$.

We further provide the choices of parameters used for the experiments in \Cref{table:potential_params}.

\begin{table}[h!]
\centering
\caption{\label{table:potential_params} Parameter values used in the implementation of each potential.}
\begin{tabular}{ll}
\toprule
\textbf{Potential Name} & \textbf{Parameter Values} \\
\midrule
Free Particle & --- \\

Barrier & $V_0 = 50.0$, $w = 0.2$ \\

Harmonic Oscillator & $m = 1.0$, $\omega = 2.0$ \\

Random Field (GRF) & $\alpha = 1, \beta=1, \gamma=4$ \\

Paul Trap & $U_0 = 10.0$, $V_0 = 15.0$, $\omega = 3.0$, $r_0 = 2.0$ \\

Shaken Lattice & $V_0 = 4.0$, $k_{\text{lat}} = 4\pi$, $A = 0.08$, $\omega_{\text{sh}} = 15.0$ \\

Gaussian Pulse & $V = 100.0$, $x_0 = 0.0$, $y_0 = 0.0$, $\sigma_x = \sigma_y = 1.2$, \\
               & $\sigma_t = 1.0$, $t_{\text{centers}} = \{0.0\}$ \\

Coulomb & $k = 1.0$, $e = 1.0$ \\

Coulomb Dipole & $V_0 = 1.0$ \\
\bottomrule
\end{tabular}
\end{table}

\section{Experiment Results Over Noise Levels}\label{appdx:exp_noise_levels}
We below present the additional results to accompany those presented in \Cref{sec:experiments}.

\begin{table}[h!]
\caption{\label{table:exp_results_01p} Average relative errors across different Hamiltonians for a relative noise level of 0.01\%, computed over 100 i.i.d. test samples, with standard deviations in parentheses. 
Note that, for the Coulomb and dipole potential, the FNO columns instead refer to SFNO models. Dashes for DeepONet and UNO indicate that they do not handle functions on a spherical domain.
}
\centering
\resizebox{\textwidth}{!}{%
\begin{tabular}{ccccc}
\toprule
{} &                    FNO &                    UNO &               DeepONet &                          Linear \\
\midrule
Barrier             &  5.111e-02 (2.543e-02) &  3.160e-02 (1.611e-02) &  1.661e-01 (8.485e-02) &  \textbf{1.955e-04 (5.894e-05)} \\
Coulomb             &  4.746e-02 (9.560e-03) &                    --- &                    --- &  \textbf{1.550e-04 (6.582e-06)} \\
Dipole              &  4.362e-02 (9.411e-03) &                    --- &                    --- &  \textbf{1.549e-04 (7.103e-06)} \\
Free                &  1.995e-02 (1.001e-02) &  1.848e-02 (6.146e-03) &  1.306e-01 (7.648e-02) &  \textbf{1.904e-04 (3.065e-05)} \\
Gaussian Pulse      &  5.024e-02 (2.948e-02) &  4.531e-02 (2.219e-02) &  2.284e-01 (1.022e-01) &  \textbf{2.012e-04 (6.267e-05)} \\
Harmonic Oscillator &  6.899e-02 (3.591e-02) &  5.559e-02 (1.578e-02) &  1.544e-01 (8.427e-02) &  \textbf{1.954e-04 (4.911e-05)} \\
Paul Trap           &  1.294e-01 (6.423e-02) &  8.915e-02 (2.669e-02) &  5.267e-01 (6.064e-02) &  \textbf{1.982e-04 (5.174e-05)} \\
Random              &  2.526e-02 (1.576e-02) &  7.334e-02 (1.656e-02) &  3.048e-01 (1.081e-01) &  \textbf{1.962e-04 (5.490e-05)} \\
Shaken Lattice      &  7.083e-02 (3.959e-03) &  1.148e-02 (4.810e-03) &  2.384e-01 (9.088e-02) &  \textbf{1.921e-04 (4.245e-05)} \\
\bottomrule
\end{tabular}
}
\end{table}

\begin{table}[h!]
\caption{\label{table:exp_results_1p} Average relative errors across different Hamiltonians for a relative noise level of 1.0\%, computed over 100 i.i.d. test samples, with standard deviations in parentheses. 
Note that, for the Coulomb and dipole potential, the FNO columns instead refer to SFNO models. Dashes for DeepONet and UNO indicate that they do not handle functions on a spherical domain.
}
\centering
\resizebox{\textwidth}{!}{%
\begin{tabular}{ccccc}
\toprule
{} &                    FNO &                    UNO &               DeepONet &                          Linear \\
\midrule
Barrier             &  5.634e-02 (2.388e-02) &  2.879e-02 (9.208e-03) &   1.54e-01 (6.566e-02) &  \textbf{1.591e-02 (1.321e-04)} \\
Coulomb             &  5.463e-02 (1.069e-02) &                    --- &                    --- &  \textbf{1.462e-02 (1.673e-04)} \\
Dipole              &  7.903e-02 (2.250e-02) &                    --- &                    --- &  \textbf{1.458e-02 (1.609e-04)} \\
Free                &  5.906e-02 (3.533e-02) &  2.294e-02 (6.852e-03) &  1.401e-01 (8.664e-02) &  \textbf{1.591e-02 (1.281e-04)} \\
Gaussian Pulse      &   6.296e-02 (2.55e-02) &  3.595e-02 (1.105e-02) &  2.547e-01 (7.022e-02) &  \textbf{1.594e-02 (1.272e-04)} \\
Harmonic Oscillator &  3.560e-02 (1.146e-02) &   3.66e-02 (1.233e-02) &  4.123e-01 (8.527e-02) &  \textbf{1.592e-02 (1.467e-04)} \\
Paul Trap           &   1.12e-01 (4.197e-02) &  9.392e-02 (2.574e-02) &  6.134e-01 (9.994e-02) &  \textbf{1.592e-02 (1.369e-04)} \\
Random              &  1.924e-02 (5.013e-03) &   2.66e-02 (6.725e-03) &  2.005e-01 (7.011e-02) &  \textbf{1.591e-02 (1.318e-04)} \\
Shaken Lattice      &  7.168e-02 (3.097e-03) &  2.905e-02 (8.274e-03) &  2.090e-01 (9.046e-02) &  \textbf{1.590e-02 (1.297e-04)} \\
\bottomrule
\end{tabular}
}
\end{table}

\newpage
\section{Additional Experiments for Partial Observation}\label{appdx:addn_par_obs}
Here we present the additional results for partial observation under a masking probability of 20\% to accompany those presented in \Cref{sec:exp_par_obs}.

\begin{table}[h!]
\caption{\label{table:par_obs_results_20p} Average relative errors across different Hamiltonians for a masking probability of 20\%, computed over 100 i.i.d. test samples, with standard deviations in parentheses. 
Note that, for the Coulomb and dipole potential, the FNO columns instead refer to SFNO models. Dashes for DeepONet and UNO indicate that they do not handle functions on a spherical domain.
}
\centering
\resizebox{\textwidth}{!}{%
\begin{tabular}{ccccc}
\toprule
{} &                    FNO &                    UNO &               DeepONet &                          Linear \\
\midrule
Barrier             &  2.282e-01 (3.934e-02) &  1.634e-01 (4.013e-02) &  4.574e-01 (1.088e-01) &  \textbf{1.596e-03 (1.508e-05)} \\
Coulomb             &  2.483e-01 (4.045e-02) &                    --- &                    --- &  \textbf{1.463e-03 (1.643e-05)} \\
Dipole              &  2.224e-01 (4.358e-02) &                    --- &                    --- &  \textbf{1.464e-03 (1.736e-05)} \\
Free                &  1.767e-01 (2.721e-02) &  1.522e-01 (2.136e-02) &  4.242e-01 (1.127e-01) &  \textbf{1.597e-03 (1.673e-05)} \\
Gaussian Pulse      &  2.796e-01 (5.107e-02) &  5.404e-01 (7.378e-02) &  3.879e-01 (1.053e-01) &  \textbf{1.598e-03 (1.965e-05)} \\
Harmonic Oscillator &  2.020e-01 (5.281e-02) &  2.082e-01 (3.244e-02) &  4.182e-01 (6.639e-02) &  \textbf{1.595e-03 (1.416e-05)} \\
Paul Trap           &  2.826e-01 (5.004e-02) &  1.000e+00 (8.714e-04) &  6.165e-01 (4.276e-02) &  \textbf{1.594e-03 (1.458e-05)} \\
Random              &  2.074e-01 (3.237e-02) &  1.889e-01 (3.114e-02) &   3.61e-01 (9.507e-02) &  \textbf{4.640e-03 (1.095e-04)} \\
Shaken Lattice      &  1.224e-01 (1.065e-02) &  1.468e-01 (1.739e-02) &  3.066e-01 (1.020e-01) &  \textbf{1.596e-03 (1.711e-05)} \\
\bottomrule
\end{tabular}
}
\end{table}

\newpage
\section{Additional Experiment for Time Generalization}\label{appdx:exp_timegen}
In Table \ref{table:time-gen-noise1}, we present additional results for time generalization under a relative noise of 1\% to accompany those presented in \Cref{sec:timegen-experiments}.
\begin{table}[H]
\caption{\label{table:time-gen-noise1}Average relative time-generalization errors across different Hamiltonians for a relative noise level of 1\%, computed over 100 i.i.d.\ test samples. Standard deviations are shown in parentheses.}
\centering
\resizebox{\textwidth}{!}{%
\begin{tabular}{lccccc}
\toprule
Hamiltonian & $j=1$ & $j=2$ & $j=4$ & $j=8$ & $j=16$ \\
\midrule
Barrier & 1.590e-02 (1.215e-04) & 2.412e-02 (2.718e-03) & 2.429e-02 (2.911e-03) & 2.326e-02 (2.419e-03) & 2.288e-02 (2.404e-03) \\
Coulomb & 1.461e-02 (1.574e-04) & 1.464e-02 (1.744e-04) & 1.459e-02 (1.559e-04) & 1.460e-02 (1.726e-04) & 1.458e-02 (1.681e-04) \\
Dipole & 1.461e-02 (1.716e-04) & 1.465e-02 (1.662e-04) & 1.461e-02 (1.646e-04) & 1.460e-02 (1.553e-04) & 1.466e-02 (1.725e-04) \\
Free & 1.593e-02 (1.493e-04) & 1.592e-02 (1.395e-04) & 1.593e-02 (1.294e-04) & 1.594e-02 (1.196e-04) & 1.590e-02 (1.300e-04) \\
Gaussian Pulse & 1.592e-02 (1.097e-04) & 2.245e-02 (3.847e-03) & 2.435e-02 (5.065e-03) & 2.466e-02 (5.202e-03) & 2.563e-02 (5.738e-03) \\
Harmonic Oscillator & 1.593e-02 (1.380e-04) & 1.592e-02 (1.248e-04) & 1.585e-02 (1.322e-04) & 1.581e-02 (1.371e-04) & 1.576e-02 (1.182e-04) \\
Paul Trap & 1.593e-02 (1.178e-04) & 1.519e-01 (2.011e-02) & 4.538e-01 (4.887e-02) & 6.345e-01 (4.912e-02) & 6.671e-01 (4.591e-02) \\
Random Lattice & 1.590e-02 (1.320e-04) & 1.590e-02 (1.349e-04) & 1.589e-02 (1.376e-04) & 1.591e-02 (1.377e-04) & 1.588e-02 (1.407e-04) \\
Shaken Lattice & 1.591e-02 (1.359e-04) & 1.677e-02 (2.702e-04) & 1.677e-02 (2.449e-04) & 1.627e-02 (1.478e-04) & 1.703e-02 (2.380e-04) \\
\bottomrule
\end{tabular}
}
\end{table}

\section{Additional Experimental Results}\label{appdx:trunc-exp}
We present in \Cref{table:trunc_exp_results} the results for the case where the test functions have a spectrum to match that of the linear estimator, i.e. where $|k|_{\infty}\le K_{n}$. Noise was disabled for this test, i.e. $\sigma=0$ was used in the data generation.

\begin{table}[h!]
\caption{\label{table:trunc_exp_results} Average relative errors across Hamiltonians assessed over a batch of 100 i.i.d. test samples with a restricted spectrum. 
}
\centering
\resizebox{\textwidth}{!}{%
\begin{tabular}{ccccc}
\toprule
{} &                    FNO &                    UNO &               DeepONet &                          Linear \\
\midrule
Barrier             &  4.231e-02 (1.711e-02) &  6.767e-02 (3.416e-02) &  2.616e-01 (7.999e-02) &  \textbf{3.085e-15 (2.219e-16)} \\
Coulomb             &  8.033e-02 (1.926e-02) &                    --- &                    --- &  \textbf{4.639e-05 (1.447e-05)} \\
Dipole              &  5.362e-02 (1.438e-02) &                    --- &                    --- &  \textbf{4.660e-05 (1.706e-05)} \\
Free                &   1.45e-02 (7.094e-03) &  2.167e-02 (1.021e-02) &  1.962e-01 (8.468e-02) &   \textbf{2.57e-15 (1.713e-16)} \\
Gaussian Pulse      &  5.412e-02 (3.173e-02) &  1.062e-01 (6.663e-02) &  2.788e-01 (8.848e-02) &  \textbf{3.129e-15 (1.119e-16)} \\
Harmonic Oscillator &   3.48e-02 (1.887e-02) &   4.439e-02 (1.78e-02) &  1.855e-01 (7.109e-02) &  \textbf{3.068e-15 (1.028e-16)} \\
Paul Trap           &  1.264e-01 (5.449e-02) &  5.246e-02 (2.538e-02) &  6.663e-01 (9.779e-02) &  \textbf{3.394e-15 (4.389e-17)} \\
Random              &  1.483e-02 (6.859e-03) &  5.872e-02 (2.138e-02) &  2.093e-01 (6.123e-02) &  \textbf{3.014e-15 (1.708e-16)} \\
Shaken Lattice      &  7.036e-02 (2.764e-03) &   1.17e-02 (4.428e-03) &  2.031e-01 (6.624e-02) &  \textbf{3.055e-15 (1.347e-16)} \\
\bottomrule
\end{tabular}
}
\end{table}

\section{Compute Resources}\label{section:compute_resources}
All experiments involving the linear estimator were run on a standard-grade CPU. The deep learning-based approaches, namely the FNO and DeepONet, were trained on an Nvidia RTX 2080 Ti.

\end{document}